%% file: example_paper.tex
\documentclass{article}

\usepackage{microtype}
\usepackage{graphicx}
\usepackage{subfigure}
\usepackage{booktabs} 

\usepackage{hyperref}

\usepackage[accepted]{icml2025}

\usepackage{amsmath}
\usepackage{amssymb}
\usepackage{mathtools}
\usepackage{amsthm}

\usepackage{amsfonts}
\usepackage{thmtools} 
\usepackage{thm-restate}

\usepackage[capitalize,noabbrev]{cleveref}

\theoremstyle{plain}
\newtheorem{theorem}{Theorem}[section]

\newtheorem{lemma}[theorem]{Lemma}
\newtheorem{corollary}[theorem]{Corollary}
\theoremstyle{definition}
\newtheorem{definition}[theorem]{Definition}
\newtheorem{assumption}[theorem]{Assumption}
\theoremstyle{remark}

\newtheorem{claim}{Claim}

\usepackage[textsize=tiny]{todonotes}

\icmltitlerunning{Error Bounds for Structured Kernel Interpolation}

\begin{document}

\twocolumn[
\icmltitle{The Price of Linear Time: Error Analysis of Structured Kernel Interpolation}



\icmlsetsymbol{equal}{*}

\begin{icmlauthorlist}
\icmlauthor{Alexander Moreno}{yyy}
\icmlauthor{Justin Xiao}{yyy}
\icmlauthor{Jonathan Mei}{yyy}

\end{icmlauthorlist}

\icmlaffiliation{yyy}{Independent Researcher}

\icmlcorrespondingauthor{Alexander Moreno}{alexander.f.moreno@gmail.com}

\icmlkeywords{Machine Learning, ICML}

\vskip 0.3in
]
\printAffiliationsAndNotice{\phantom{}}




\begin{abstract}
Structured Kernel Interpolation (SKI) \cite{wilson2015kernel} helps scale Gaussian Processes (GPs) by approximating the kernel matrix via interpolation at inducing points, achieving linear computational complexity. However, it lacks rigorous theoretical error analysis. This paper bridges the gap: we prove error bounds for the SKI Gram matrix and examine the error's effect on hyperparameter estimation and posterior inference. We further provide a practical guide to selecting the number of inducing points under convolutional cubic interpolation: they should grow as $n^{d/3}$ for error control. Crucially, we identify two dimensionality regimes governing the trade-off between SKI Gram matrix spectral norm error and computational complexity. For $d \leq 3$, \textit{any} error tolerance can achieve linear time for sufficiently large sample size. For $d > 3$, the error must \textit{increase} with sample size to maintain linear time. Our analysis provides key insights into SKI's scalability-accuracy trade-offs, establishing precise conditions for achieving linear-time GP inference with controlled approximation error.
\end{abstract}

\section{Introduction}\label{sec:introduction}
\input{body/01intro}

\section{Related Work}\label{sec:related}
\input{body/02related}

\section{Gaussian Processes, Structured Kernel Interpolation and Convolutional Cubic Interpolation}\label{sec:ski-background}
\input{body/03background}

\section{Important Quantities}\label{sec:important-quantities}

\input{body/04important-quantities}

\section{Gaussian Processes Applications}\label{sec:gp-applications}
\input{body/05gp-applications}

\section{Discussion}\label{sec:discussion}
\input{body/06discussion}

\section{Broader Impact}\label{sec:broader-impact}
\input{body/07broader-impact}

\bibliography{example_paper}
\bibliographystyle{icml2025}

\newpage
\appendix
\onecolumn

\section{Auxiliary Technical Results}
\input{body/appendices/01auxiliary}
\section{Proofs Related to Important Quantities}
\input{body/appendices/04important-quantities}

\section{Proofs Related to Gaussian Process Applications}
\input{body/appendices/05gp-applications}

\end{document}

%% file: body/01intro.tex
Gaussian Processes (GPs) \citep{kolmogorov1940wienersche,rasmussen2006gaussian} are an important class of stochastic processes used in machine learning and statistics, with use cases including spatial data analysis \citep{liu2021missing}, time series forecasting \citep{girard2002gaussian}, bioinformatics \citep{luo2023diseasegps} and Bayesian optimization \citep{frazier2018tutorial}. GPs offer a non-parametric framework for modeling distributions over functions, enabling both flexibility and uncertainty quantification. These capabilities, combined with the ability to incorporate prior knowledge and specify relationships by choice of kernel function, make Gaussian Processes effective for both regression and classification.

However, GPs have substantial computational and memory bottlenecks. Both training and inference require computing the action of the inverse kernel Gram matrix, while training requires computing its log-determinant: both are $O(n^3)$ operations with sample size $n$. Further, storing the full Gram matrix requires $O(n^2)$ memory. These bottlenecks require scalable approximations for larger datasets.

Structured Kernel Interpolation (SKI) \cite{wilson2015kernel} helps scale Gaussian Processes (GPs) to large datasets by approximating the kernel matrix using interpolation on a set of inducing points. For stationary kernels, this requires $O(n+m \log m)$ computational complexity. The core idea is to express the original kernel as a combination of interpolation functions and a kernel matrix defined on a set of inducing points. However, despite its effectiveness, popularity (over $600$ citations, a large number for a GP paper) and high quality software availability (\cite{gardner2018gpytorch} has 3.5k stars on github), it currently lacks theoretical analysis. A key initial question is, given a fixed error bound for the SKI Gram matrix and use of cubic convolutional interpolation, how many inducing points are required to achieve that error bound? Given the required value of $m$ as a function of $n$, for what error tolerance is $O(n+m\log m)$ still linear? Following this, what do these errors imply for hyperparameter estimation and posterior inference?

\begin{table*}[h]
\centering
\begin{tabular}{|l|l|}
\hline
\textbf{Quantity} & \textbf{Bound} \\
\hline
SKI kernel error & $O(\frac{c^{2d}}{m^{3/d}})$ \\
\hline
SKI Gram matrix error & $O(\frac{nc^{2d}}{m^{3/d}})$ \\
\hline
SKI cross-kernel matrix error & $O(\frac{\max(n,T)c^{2d}}{m^{3/d}})$ \\
\hline
SKI score function error & $O(\frac{\sqrt{p}n^{2}c^{4d}}{m^{3/d}})$ \\
\hline
SKI posterior mean error & $O(c^{2d}\frac{\max(T,n)+\sqrt{Tn}n}{m^{3/d}})$ \\
\hline
SKI posterior covariance error & $O(\frac{Tn^{2}mc^{4d}+\sqrt{Tn}mc^{4d}\max(T,n)}{m^{3/d}})$ \\
\hline
\end{tabular}
\caption{Summary of Theoretical Results when using SKI with convolutional cubic interpolation. This shows the rate at which the error of using SKI (vs the exact kernel) grows as a function of important variables. Here $n$ and $T$ are the train/test sample sizes, $d$ is the dimensionality, $m$ the number of inducing points, $p$ is the number of hyperparameters and $c>0$ is a constant. Most importantly, the Gram matrix error grows linearly with the sample size, exponentially with the dimension while decaying at an $m^{3/d}$ rate in the inducing points.}
\label{table:theoretical_results}
\end{table*}

In this paper, we begin to bridge the gap between practice and a theoretical understanding of SKI. We have three primary contributions: 1) The first error analysis for the SKI kernel and relevant quantities, including the SKI gram matrix's spectral norm error. Based on this we provide \textit{a practical guide to select the number of inducing points}: they should grow as $n^{d/3}$ to control error. 2) SKI hyperparameter estimation analysis. 3) SKI inference analysis: the error of the GP posterior means and variances at test points. We find two interesting results: 1) we identify two dimensionality regimes relating SKI Gram matrix error to computational complexity. For $d\leq 3$, for \textit{any} fixed spectral norm error, we can achieve it in linear time using SKI with a sufficient sample size. For $d>3$, the error must \textit{increase} with the sample size to maintain our guarantee of linear time. 2) For a $\mu$-smooth log-likelihood, gradient ascent on the SKI log-likelihood will approach a neighborhood of a stationary point of the true log-likelihood at a $O\left(\frac{1}{K}\right)$ rate, with the neighborhood size determined by the SKI score function's error, which aside from the response variables grows \textit{linearly} with the sample size when increasing inducing points as we suggested. To obtain this, we leverage a recent result \cite{stonyakin2023stopping} from the inexact gradient descent \cite{daspremont2008smooth,devolder2014first} literature.


In section \ref{sec:related} we describe related work. In section \ref{sec:ski-background} we give a brief background on SKI. In section \ref{sec:important-quantities} we bound the error of important quantities: specifically the SKI kernel, Gram matrix and cross-kernel matrix errors. In section \ref{sec:gp-applications} we use these to analyze the error of the SKI MLE and posteriors. We conclude in section \ref{sec:discussion} by summarizing our results and discussing limitations and future work.

%% file: body/02related.tex
We can divide related works into three groups: those theoretically analyzing Gaussian process regression or kernel methods when using approximate kernels, SKI and its extensions, and papers developing techniques we use to obtain our guarantees. In the first group, the most relevant works are \cite{burt2019rates,burt2020convergence}, where they analyzed the sparse variational GP framework \citep{titsias2009variational,hensman2013gaussian} and derived bounds on the Kullback-Leibler divergence between the true posterior and the variational approximate posterior. \cite{moreno2023ski} gave bounds on the approximation error of the SKI Gram matrix. However, they only handled the case of univariate features and only bounded how much worse the SKI Gram matrix can be than the Nystr{\"o}m one. Further, they did not analyze the downstream effects on the approximate MLE or GP posterior. Also relevant are \cite{wynne2022variational,wild2021connections}, who gave a Banach space view of sparse variational GPs and connected them to the Nystr{\"o}m method, respectively. Finally, \cite{modell2024entrywise} provide entry-wise error bounds for low-rank approximations of kernel matrices: our approach also relies on entry-wise error bounds, but theirs are for the \textit{best} low-rank approximation to a given gram matrix, while ours are for the SKI gram matrix. Only one of these papers \cite{moreno2023ski} treated SKI specifically, and it only covered a very special case setting.

In the second group, the foundational work by \cite{wilson2015kernel} that we analyze introduced SKI as a scalable method for large-scale GP inference. \cite{kapoor2021skiing} extended SKI to high-dimensional settings using the permutohedral lattice. \cite{yadav2022kernel} developed a sparse grid approach to kernel interpolation that also helps address the curse of dimensionality. Most recently, \cite{ban2024malleable} proposed a flexible adaptation of SKI with a hyperparameter that adjusts the number of grid points based on kernel hyperparameters. We focus our analysis on the original technique of \cite{wilson2015kernel} in this paper, but future work could extend to the settings of the latter papers.

Also relevant are papers where we leverage or extend their results and proof techniques. We require a multivariate extension to the error analysis of \cite{keys1981cubic} for convolutional cubic interpolation, which we derive. We also use a recent result from the inexact gradient descent literature \cite{stonyakin2023stopping}, which allows us to analyze the effect of doing gradient ascent on the SKI log-likelihood instead of the true log-likelihood. Finally, we use a proof technique \cite{bach2013sharp,musco2017recursive} commonly used to bound the in-sample error of approximate kernel ridge regression to bound the test SKI mean function error.


%% file: body/03background.tex
This section provides background on Gaussian Processes (GPs) and two key techniques for enabling scalable inference: Structured Kernel Interpolation (SKI) and Convolutional Cubic Interpolation. SKI \cite{wilson2015kernel} addresses GPs scalability issue by approximating the kernel matrix through interpolation on a set of inducing points, leveraging the efficiency of convolutional kernels. In particular, cubic convolutional kernels, as detailed in \cite{keys1981cubic}, provide a smooth and accurate interpolation scheme that forms the foundation of the SKI framework. In this paper, we focus on this cubic case as it is used by SKI. Future work may extend this to study higher-order interpolation methods. Here, we formally define these concepts and lay the groundwork for the subsequent error analysis.

\subsection{Gaussian Processes}\label{sec:gp-background}
A Gaussian process $\xi\sim \textrm{GP}(\nu,k_\theta)$ is a stochastic process $\{\xi(\textbf{x})\}_{\textbf{x}\in \mathcal{X}}$ such that any finite subcollection $\{\xi(\textbf{x}_i)\}_{i=1}^n$ is multivariate Gaussian distributed. We assume that we have index locations $\textbf{x}_i\in \mathbb{R}^d$ and observations $y_i\in \mathbb{R}$ for a set of training points $i=1,\ldots,n$ such that
\begin{align*}
y_i&=\xi(\textbf{x}_i)+\epsilon_i,\epsilon_i\sim \mathcal{N}(0,\sigma^2).
\end{align*}
where $\nu:\mathcal{X}\rightarrow\mathbb{R}$, $k_\theta:\mathcal{X}\times \mathcal{X}\rightarrow\mathbb{R}$ are the prior mean and covariance functions, respectively, with $k$ an SPD kernel with hyperparameters $\theta$. Given $\{\textbf{x}_i,y_i\}_{i=1}^n$ we are primariliy interested in two tasks: 1) estimate hyperparameters $\boldsymbol{\theta}\in \Theta\subseteq \mathbb{R}^p$ of kernel $k_\theta$ (e.g. RBF kernel) 2) do Bayesian inference for the posterior mean $\boldsymbol{\mu}(\cdot)\in \mathbb{R}^{T}$ and covariance $\boldsymbol{\Sigma}(\cdot)\in \mathbb{R}^{T\times T}$ at a set of test points $\{\textbf{x}_t\}_{t=1}^T$. Assuming $\nu\equiv 0$ (a mean-zero GP prior), for 1), one maximizes the log-likelihood
\begin{align}
    \mathcal{L}(\boldsymbol{\theta};X) &= -\frac{1}{2}\textbf{y}^\top (\textbf{K}+\sigma^2 \textbf{I})^{-1}\textbf{y}\nonumber\\
    &\qquad-\frac{1}{2}\log \vert \textbf{K}+\sigma^2 \textbf{I}\vert -\frac{n}{2}\log (2\pi)\label{eqn:log-likelihood}
\end{align}
to find $\boldsymbol{\theta}\in \mathcal{D}\subseteq \Theta$ where $\textbf{K}\in \mathbb{R}^{n\times n}$ with entries $\textbf{K}_{ij}=k_\theta(\textbf{x}_i,\textbf{x}_j)$ is the Gram matrix for the training dataset. For 2), given the kernel function and known observation variance $\sigma^2$, the posterior mean and covariance are given by
\begin{align}
\boldsymbol{\mu}(\cdot) &= \mathbf{K}_{\cdot, \mathbf{X}}\left(\mathbf{K}+\sigma^{2} \mathbf{I}\right)^{-1} \mathbf{y}\label{eqn:posterior-mean}\\
\boldsymbol{\Sigma}(\cdot) &= \textbf{K}_{\cdot,\cdot}+\sigma^2I-\textbf{K}_{\cdot,\textbf{X}} (\textbf{K}+\sigma^2 I)^{-1}\textbf{K}_{\textbf{X},\cdot}\label{eqn:posterior-covariance}
\end{align}
where $\textbf{K}_{\cdot,\textbf{X}}\in \mathbb{R}^{T\times n}$ is the matrix of kernel evaluations between test and training points. Intuitively, the GP prior represents our belief about all possible functions before seeing any data. When we observe data points, the posterior represents our updated belief - it gives higher probability to functions that fit our observations while maintaining the smoothness properties encoded in the kernel. The posterior mean can be viewed as a weighted average of these functions, where the weights depend on how well each function fits the data and satisfies the prior assumptions. The posterior variance indicates our remaining uncertainty - it is smaller near observed points where we have more confidence, and larger in regions far from our data.

A challenge is that, between the log-likelihood and the posteriors, one first needs to compute the action of the inverse of the regularized Gram matrix, $(\textbf{K}+\sigma^2 \textbf{I})^{-1}\textbf{y}$. Second, one needs to compute the log-determinant $\log \vert \textbf{K}+\sigma^2 \textbf{I}\vert$. These are both $O(n^3)$ computationally and $O(n^2)$ memory.

\subsection{Structured Kernel Interpolation}\label{subsec:ski}
Structured kernel interpolation \citep{wilson2015kernel} or (SKI) addresses these computational and memory bottlenecks by approximating the original kernel function \(k_\theta:\mathcal{X}\times \mathcal{X}\rightarrow \mathbb{R},\mathcal{X}\subseteq\mathbb{R}^d\) by interpolating kernel values at a chosen set of inducing points \(\mathbf{U}=\left(\begin{array}{c}
     \textbf{u}_1^\top\\
     \vdots\\
     \textbf{u}_m^\top
\end{array}\right)\in \mathbb{R}^{m\times d}\). The approximate kernel function \(\tilde{k}:\mathcal{X}\times \mathcal{X}\rightarrow\mathbb{R}\) can be expressed as:
\[
\tilde{k}(\mathbf{x}, \mathbf{x}') = \mathbf{w}(\mathbf{x})^{\top} \mathbf{K}_{\mathbf{U}} \mathbf{w}(\mathbf{x}')
\]
where \(\mathbf{K}_{\mathbf{U}}\in \mathbb{R}^{m\times m}\) is the kernel matrix computed on the inducing points, and \(\mathbf{w}(\mathbf{x}),\mathbf{w}(\mathbf{x}')\in \mathbb{R}^m\) are vectors of interpolation weights using (usually cubic) convolutional kernel $u:\mathbb{R}\rightarrow \mathbb{R}$ for the points $\textbf{x}$ and $\textbf{x}'$, respectively. One can then form the SKI Gram matrix $\tilde{\textbf{K}}=\textbf{W}\textbf{K}_\textbf{U}\textbf{W}^\top$ with $\textbf{W}$ a \textit{sparse} matrix of $L$ interpolation weights per row for a polynomial of degree $L-1$. By exploiting the sparsity of each row, for stationary kernels this leads to a computational complexity of $O(nL+m\log m)$ and a memory complexity of $O(nL+m)$.

In order to learn kernel hyperparameters, one can maximize the SKI approximation to the log-likelihood (henceforth the SKI log-likelihood)
\begin{align*}
    \tilde{\mathcal{L}}(\boldsymbol{\theta};X)
&=-\frac{1}{2}\textbf{y}^\top (\tilde{\textbf{K}}+\sigma^2 \textbf{I})^{-1}\textbf{y}\\
&\qquad-\frac{1}{2}\log \vert \tilde{\textbf{K}}+\sigma^2 \textbf{I}\vert -\frac{n}{2}\log (2\pi)
\end{align*}
Given the SKI kernel $\tilde{k}:\mathcal{X}\times \mathcal{X}\rightarrow \mathbb{R}$ with learned hyperparameters, one can do posterior inference of the SKI approximations to the mean $\tilde{\boldsymbol{\mu}}(\cdot)$ and covariance $\tilde{\boldsymbol{\Sigma}}(\cdot)$ at a set of $T$ test points $\cdot$ as
\begin{align*}
    \tilde{\boldsymbol{\mu}}(\cdot)&=\tilde{\mathbf{K}}_{\cdot, \mathbf{X}}\left(\tilde{\mathbf{K}}+\sigma^{2} \mathbf{I}\right)^{-1} \mathbf{y}\\
    \tilde{\boldsymbol{\Sigma}}(\cdot)&=\tilde{\textbf{K}}_{\cdot,\cdot}+\sigma^2I-\tilde{\textbf{K}}_{\cdot,\textbf{X}} (\tilde{\textbf{K}}+\sigma^2 I)^{-1}\tilde{\textbf{K}}_{\textbf{X},\cdot}
\end{align*}
where $\tilde{\textbf{K}}_{\cdot,\textbf{X}}\in \mathbb{R}^{T\times n}$ is the matrix of SKI kernels between test points and training points and $\tilde{\textbf{K}}_{\cdot,\cdot}\in \mathbb{R}^{T\times T}$ is the SKI Gram matrix for the test points. Going forward, we may write $\mathcal{L}(\boldsymbol{\theta})$ and $\tilde{\mathcal{L}}(\boldsymbol{\theta})$, dropping the explicit dependence on the data but implying it.

\subsection{Convolutional Cubic Interpolation}\label{subsec:convolutional-cubic-interpolation}

Convolutional cubic interpolation \citep{keys1981cubic} gives a continuously differentiable interpolation of a function given its values on a regular grid, where its cubic convolutional kernel is a piecewise polynomial function designed to ensure continuous differentiability. We formalize this using the definitions of the cubic convolutional interpolation kernel and the tensor-product cubic convolutional function below. We also define an upper bound for the sum of weights for each dimension, which will be a useful constant going forward. Such a bound will exist for all continuous stationary kernels vanishing at infinity.

\begin{definition}\label{def:cubic-interpolation-kernel}

The cubic convolutional interpolation kernel $u:\mathbb{R}\rightarrow\mathbb{R}$ is given by
$$
u(s)\equiv\begin{cases}
1, & s=0\\
\frac{3}{2}\vert s\vert^3-\frac{5}{2}\vert s\vert^2+1, & 0<\vert s\vert <1\\
-\frac{1}{2}\vert s\vert^3+\frac{5}{2}\vert s\vert^2-4\vert s\vert+2, & 1<\vert s\vert<2\\
0, & \text{otherwise}
\end{cases}
$$
\end{definition}

\begin{definition}\label{def:tensor-product-cubic-interpolation-alt1}
Let $\mathbf{x} = (x_1, x_2, ..., x_d) \in \mathbb{R}^d$ be a d-dimensional point. Let $f:\mathbb{R}^d \rightarrow \mathbb{R}$ be a function defined on a regular grid with spacing $h$ in each dimension. Let $\mathbf{c_x}$ denote the grid point closest to $\mathbf{x}$.  The tensor-product cubic convolutional interpolation function $g:\mathbb{R}^d\rightarrow \mathbb{R}$ is defined as:
{\footnotesize
\begin{align*}
    g(\mathbf{x}) \equiv \sum_{\mathbf{k} \in \{-1, 0, 1, 2\}^d} f(\mathbf{c_x} + h\mathbf{k}) \prod_{j=1}^d u\left(\frac{x_j - (\mathbf{c_x})_j - h k_j}{h}\right)
\end{align*}
}
where $u$ is the cubic convolutional interpolation kernel and $\mathbf{k} = (k_1, \ldots, k_d)$ is a vector of integer indices.
\end{definition}

\begin{definition}\label{def:sum-weight-upper-bound}
    Given an interpolation kernel $u:\mathbb{R}\rightarrow\mathbb{R}$ and a fixed $n\in \mathbb{N}$, let $c>0$ be an upper bound such that, for any $x\in  \mathbb{R}$ and a set of data points $\{x_i\}_{i=1}^n \subset \mathbb{R}$,
$$
\sum_{i=1}^n \left\vert u\left(\frac{x - x_i}{h}\right) \right\vert\leq c,
$$
\end{definition}

Going forward, we always assume that we use convolutional cubic polynomial interpolation, so that $L=4$ as in \cite{wilson2015kernel}, but that we may vary the number of inducing points $m$. In particular, we will analyze how the number of inducing points affects error for different terms of interest, and how to choose the number of inducing points.

%% file: body/04important-quantities.tex
This section derives bounds for key quantities in Structured Kernel Interpolation (SKI). Section \ref{subsubsec:elementwise-error} provides a bound on the elementwise error between the true kernel and its SKI approximation. In Section \ref{subsubsec:spectral-norm-error}, we extend this to the spectral norm error of the SKI approximation for the training Gram matrix and train-test kernel matrix. Finally, in section \ref{subsec:linear-time} we present conditions on the number of inducing points for achieving specific error tolerance $\epsilon>0$ and error needed to guarantee linear time complexity, noting linear time always holds for $d\leq 3$ with sufficiently large samples.

\subsection{Error Bounds for the Ski Kernel}\label{subsec:ski-kernel-error-bounds}
This subsection analyzes the error introduced by the SKI approximation of the kernel function. We start by extending the analysis of \cite{keys1981cubic} to the multivariate setting, deriving error bounds for multivariate cubic convolutional polynomial interpolation. We then use these to derive the elementwise error for the SKI approximation \(\tilde{k}(\textbf{x},\textbf{x}')\). We next apply these elementwise bounds to derive spectral norm error bounds for SKI kernel matrices, which will be crucial for understanding the downstream effects of the SKI approximation on Gaussian process hyperparameter estimation and posterior inference.

\subsubsection{Elementwise}\label{subsubsec:elementwise-error}

Our first lemma shows that multivariate tensor-product cubic convolutional interpolation retains error cubic in the grid spacing of \cite{keys1981cubic}, which is equivalent to $m^{-3/d}$ decay with the number of inducing points $m$, but exhibits exponential error growth with increasing dimensions. The proof uses induction on dimensions, starting with the 1D case from Keys.

\begin{restatable}{lemma}{tensorproductinterpolationerror}\label{lemma:tensor-product-interpolation-error}
The error of tensor-product cubic convolutional interpolation is $O(c^d h^3)$, or equivalently $O\left(\frac{c^d}{m^{3/d}}\right)$.
\end{restatable}
\begin{proof}
    See Appendix \ref{sec:proof-tensor-product-interpolation}.
\end{proof}

The following Lemma allows us to bound the absolute difference between the true and SKI kernels \textit{uniformly} with the same big-$O$ error as for the underlying interpolation itself. The proof uses the the triangle inequality to decompose the error into two parts: the first is the error from a single interpolation, while the second is the error of the nested interpolations.

\begin{restatable}{lemma}{skikernelelementwiseerror}\label{lemma:ski-kernel-elementwise-error}
    Let $\delta_{m,L}$ be the interpolation error for $m$ inducing points and interpolation degree $L-1$. The SKI kernel $\tilde{k}:\mathcal{X}\times \mathcal{X}\rightarrow \mathbb{R}$ with grid spacing $h$ in each dimension has error
    \begin{align*}
        \vert k(\textbf{x},\textbf{x}')-\tilde{k}(\textbf{x},\textbf{x}')\vert&= \delta_{m,L}+\sqrt{L}c^d\delta_{m,L}\\
        &=O\left(\frac{c^{2d}}{m^{3/d}}\right).
    \end{align*}

\end{restatable}
\begin{proof}
    See Appendix \ref{sec:proof-ski-kernel-elementwise-error}
\end{proof}

\subsubsection{Spectral Norm Error}\label{subsubsec:spectral-norm-error}

We now transition from elementwise error bounds to spectral norm bounds for the SKI gram matrix's approximation error, finding that it grows linearly with the sample size and exponentially with the dimension and decays as $m^{-3/d}$ with the number of inducing points. This is both of independent interest but will also be important to nearly all downstream analysis for estimation and inference. We also provide a bound on the spectral norms of the SKI train/test kernel matrix's approximation error. This is useful when analyzing the GP posterior parameter error.

For this next lemma we will express it both in the general interpolation setting and again give the specific big-$O$ for convolutional cubic interpolation, but going forward we sometimes only show the latter setting in the main paper and derive the general settings in the proof. In particular, \emph{\textbf{whenever we use big $O$-notation}} we are assuming convolutional cubic interpolation.
\begin{restatable}{proposition}{spectralnorm}\label{prop:spectral-norm}
    For the SKI approximation $\tilde{\textbf{K}}$ of the true Gram matrix $\textbf{K}$, we have
    \begin{align*}
        \Vert \textbf{K}-\tilde{\textbf{K}}\Vert_2&=n \left(\delta_{m,L}+\sqrt{L} c^d\delta_{m,L}\right)\\
        &\equiv \gamma_{n,m,L}\\
        &= O\left(\frac{nc^{2d}}{m^{3/d}}\right)
    \end{align*}
\end{restatable}
\begin{proof}
    See Appendix \ref{sec:proof-spectral-norm}
\end{proof}

\begin{restatable}{lemma}{testtrainkernelmatrixerror}\label{lemma:test-train-kernel-matrix-error}
    Let \(\textbf{K}_{\cdot,\textbf{X}}\in \mathbb{R}^{T\times n}\) be the matrix of kernel evaluations between \(T\) test points and \(n\) training points, and let \(\tilde{\textbf{K}}_{\cdot,\textbf{X}}\in \mathbb{R}^{T\times n}\) be the corresponding SKI approximation. Then
    \begin{align*}
        \Vert \textbf{K}_{\cdot,\textbf{X}}-\tilde{\textbf{K}}_{\cdot,\textbf{X}}\Vert_2=O\left(\frac{\max(n,T)c^{2d}}{m^{3/d}}\right)
    \end{align*}

\end{restatable}
\begin{proof}
    See Appendix \ref{sec:proof-test-train-kernel-matrix-error}.
\end{proof}

\subsection{Achieving Errors in Linear Time}\label{subsec:linear-time}

Here, we show how many inducing points $m$ are sufficient to achieve a desired error tolerance $\epsilon > 0$ for the SKI Gram matrix when using cubic convolutional interpolation. Based on the Theorem, we should grow the number of inducing points at an $n^{d/3}$ rate. We then show corollaries describing 1) how $\epsilon$ and $m$ must grow to maintain linear time 2) how the dimension affects whether the error must grow with the sample size to ensure linear time SKI.

The following theorem shows the number of inducing points that will guarantee a Gram matrix error tolerance. It says that the number of inducing points should grow as $n^{d/3}$ to achieve a fixed error. The proof starts by lower bounding the desired spectral norm error with the upper bound on the actual spectral norm error derived in Proposition \ref{prop:spectral-norm}: this is a sufficient condition for the desired spectral norm error to hold. It then relates the number of inducing points to the grid spacing in the SKI approximation, assuming a regular grid with equal spacing in each dimension. By substituting this relationship into the sufficient condition, the proof derives the sufficient number of inducing points to control error. 

\begin{restatable}{theorem}{inducingpointscountalt}\label{thm:inducing-points-count-alt}
    If the domain is $[-D, D]^d$, then to achieve a spectral norm error of $\Vert \textbf{K} - \tilde{\textbf{K}} \Vert_2 \leq \epsilon$, it is sufficient to choose the number of inducing points $m$ such that:
    $$
    m = \left( \frac{n}{\epsilon} (1 + 2c^d) K' (8 c^{2d} D^3) \right)^{d/3}
    $$
    for some constant $K'$ that depends only on the kernel function and the interpolation scheme.
\end{restatable}
\begin{proof}
    See Appendix \ref{sec:proof-inducing-points-count-alt}.
\end{proof}
This result shows that the number of inducing points should grow
\begin{itemize}
\item Sub-linearly with the sample size and decrease in error for $d<3$, linearly for $d=2$ and super-linearly for $d>3$. Thus, as we want a tighter error tolerance or have more observations we need more inducing points, but at very different rates depending on the dimensionality.
\item Linearly with the volume of the domain $(2D)^d$. Thus, if our observations are concentrated in a small region and we select an appropriately sized domain to cover it we need fewer inducing points.
\item Exponentially with the dimension $d$, as we have a $c^{2d}$ term.
\end{itemize}

The next Corollary establishes a condition on the spectral norm error, $\epsilon$, that ensures linear-time $O(n)$ computational complexity for SKI. The core idea is that $\epsilon$ should be such that if we choose $m$ based on the previous Theorem, $m=O(n/\log n)$ and thus $m\log m=O(n)$. 
    
\begin{restatable}{corollary}{corlineartime}\label{cor:linear-time}
    If
\begin{equation} \label{eq:epsilon_condition}
\epsilon \geq \frac{(1 + 2c^d) K' 8 c^{2d} D^3}{C^{3/d}} \cdot \frac{n (\log n)^{3/d}}{n^{3/d}}
\end{equation}
for some constants $K,C>0$ that depend on the kernel function and the interpolation scheme and we choose $m>0$ based on the previous theorem, then we have both $\Vert \textbf{K}-\tilde{\textbf{K}}\Vert_2\leq \epsilon$ and
SKI computational complexity of $O(n)$.
\end{restatable}
\begin{proof}
    See Appendix \ref{proof:cor-linear-time}.
\end{proof}

Interestingly, the previous Theorem and Corollary implies a fundamental difference between two dimensionality regimes. For $d \leq 3$, the choice of $m$ required for a fixed $\epsilon$ grows more slowly than $n/\log n$. This means that for any fixed $\epsilon > 0$, SKI with cubic interpolation is guaranteed to be a linear-time algorithm for sufficiently large $n$. In contrast, for $d > 3$, the choice of $m$ required for a fixed $\epsilon>0$ eventually grows faster than $n/\log n$. Thus, to maintain linear-time complexity for $d > 3$ and the guarantees from Theorem \ref{thm:inducing-points-count-alt}, we must allow the error $\epsilon$ to increase with $n$. This demonstrates that the curse of dimensionality impacts the scalability of SKI, making it challenging to achieve both high accuracy and linear-time complexity in higher dimensions. The next corollary formalizes this.
\begin{corollary}
    For $d\leq 3$, for any $\epsilon>0$, Corollary \ref{cor:linear-time} holds for any $n$ sufficiently large, so that choosing $m$ based on Theorem \ref{thm:inducing-points-count-alt} is sufficient to achieve linear complexity. For $d>3$, $\epsilon$ must grow with the sample size to maintain linear complexity.
\end{corollary}
\begin{proof}
    For $d\leq3 $, the RHS of Eqn. \ref{eq:epsilon_condition} decreases with $n$ with limit $0$ and thus for sufficiently large sample size will be $\leq \epsilon$, satisfying the conditions to guarantee small error and linear time. For $d>3$, the RHS of Eqn. \ref{eq:epsilon_condition} grows with $n$, so that $\epsilon$ must grow to satisfy the conditions for the guarantee. 
\end{proof}

%% file: body/05gp-applications.tex
In this section, we address how SKI affects Gaussian Processes Applications. In Section \ref{sec:kernel-hyperparameter-estimation} we address how using the SKI kernel and log-likelihood affect hyperparameter estimation, showing that gradient ascent on the SKI log-likelihood approaches a ball around a stationary point of the true log-likelihood. In section \ref{sec:posterior-inference} we describe how using SKI affects the accuracy of posterior inference.

\subsection{Kernel Hyperparameter Estimation}\label{sec:kernel-hyperparameter-estimation}

Here we show that, for a $\mu$-smooth log-likelihood, an iterate of gradient ascent on the SKI log-likelihood approaches a neighborhood of a stationary point of the true log-likelihood at an $O\left(\frac{1}{K}\right)$ rate, with the neighborhood size determined by the SKI score function's error. To show this, we leverage a recent result for non-convex inexact gradient ascent \cite{stonyakin2023stopping}, which requires an upper bound on the SKI score function's error. This requires bounding the spectral norm error of the SKI Gram matrix's partial derivatives. In order to obtain this, we note that for many SPD kernels, under weak assumptions, the partial derivatives are \textit{also} SPD kernels, and thus we can reuse the previous results directly on the partial derivatives.

Note that \cite{stonyakin2023stopping} does not actually imply \textit{convergence} to a neighborhood of a critical point, only that at least one iterate will approach it. Given the challenges of non-concave optimization and the fact that we leverage a fairly recent result, we leave stronger results to future work.


Let $\mathcal{D}\subseteq \Theta$ be a \textit{compact} subset that we wish to optimize over. In the most precise setting we would analyze projected gradient ascent, but for simplicity we analyze gradient ascent. Let let $\tilde{k}_{\theta}: \mathcal{X} \times \mathcal{X} \rightarrow \mathbb{R}$ be the SKI approximation of $k_{\theta}: \mathcal{X} \times \mathcal{X} \rightarrow \mathbb{R}$ using $m$ inducing points and interpolation degree $L-1$. We are interested in the convergence properties of \textit{inexact gradient ascent} using the SKI log-likelihood, e.g.
\begin{align*}
    \boldsymbol{\theta}_{k+1}&=\boldsymbol{\theta}_k+\eta \nabla \tilde{\mathcal{L}}(\boldsymbol{\theta}_k),
\end{align*}
where $\eta\in \mathbb{R}$ is the learning rate and $\nabla \tilde{\mathcal{L}}(\boldsymbol{\theta}_k)$ is the SKI score function (gradient of its log-likelihood). We assume: 1) a $\mu$-smooth log-likelihood. If we optimize on a bounded domain, then for infinitely differentiable kernels (e.g. RBF) this will immediately hold. 3) That the kernel's partial derivatives are themselves SPD kernels (this can be easily shown for the RBF kernel's lengthscale by noting that the product of SPD kernels are themselves SPD kernels). 
\begin{assumption}[$\mu$-smooth-log-likelihood]\label{assumption:mu-smoothness}
The true log-likelihood is $\mu$-smooth over $\mathcal{D}$. That is, for all $\boldsymbol{\theta},\boldsymbol{\theta}'\in \mathcal{D}$,
\begin{align*}
    \Vert \nabla \mathcal{L}(\boldsymbol{\theta})-\nabla \mathcal{L}(\boldsymbol{\theta}')\Vert &\leq \mu \Vert \boldsymbol{\theta}-\boldsymbol{\theta}'\Vert
\end{align*}
\end{assumption}

\begin{assumption}
    (Kernel Smoothness) $k_\theta(x,x')$ is $C^1$ in $\boldsymbol{\theta}$ over $\mathcal{D}$. That is, for each $l \in \{1, ..., p\}$, $k'_{\theta_l}(x, x') = \frac{\partial k_{\theta}(x, x')}{\partial \theta_l}$ exists and is continuous for $\boldsymbol{\theta}\in \mathcal{D}$.
\end{assumption}
\begin{assumption}
    (SPD Kernel Partials) For each $l \in \{1, ..., p\}$, the partial derivative of $k_{\theta}$ with respect to a hyperparameter $\theta_l\in \mathbb{R}$, denoted as $k'_{\theta_l}(x, x') = \frac{\partial k_{\theta}(x, x')}{\partial \theta_l}$, is also a valid SPD kernel.
\end{assumption}

We next state several results leading up to our bound on the SKI score function's error. Here we argue that we can apply the same elementwise error we derived previously to the SKI partial derivatives.

\begin{restatable}{lemma}{skikernelderivativeerrorkernel}[Bound on Derivative of SKI Kernel Error using Kernel Property of Derivative]
\label{lemma:ski_kernel_derivative_error_kernel}
 Let $\tilde{k}'_{\theta_l}(x,x')$ be the SKI approximation of $k'_{\theta_l}(x,x')$, using the same inducing points and interpolation scheme as $\tilde{k}_{\theta}$. Then, for all $x, x' \in \mathcal{X}$ and all $\boldsymbol{\theta} \in \Theta$, the following inequality holds:

\begin{align*}
\left\vert \frac{\partial k_{\theta}(x,x')}{\partial \theta_l}-\frac{\partial \tilde{k}_{\theta}(x,x')}{\partial \theta_l}\right\vert &= \left\vert k'_{\theta_l}(x, x') - \tilde{k}'_{\theta}(x, x') \right\vert \\
&\leq\delta_{m,L}'+\sqrt{L}c^d\delta_{m,L}'\\
&=O\left(\frac{c^{2d}}{m^{3/d}}\right)
\end{align*}

where $\delta_{m,L}'$ is an upper bound on the error of the SKI approximation of the kernel $k'_{\theta_l}(x,x')$ with $m$ inducing points and interpolation degree $L-1$, as defined in Lemma \ref{lemma:ski-kernel-elementwise-error}.
\end{restatable}
\begin{proof}
    See Appendix \ref{sec:proofski_kernel_derivative_error_kernelz}
\end{proof}

We then use the elementwise bound to bound the spectral norm of the SKI gram matrix's partial derivative error. This again leverages Proposition \ref{prop:spectral-norm}, noting that these partial derivatives of the Gram matrices are themselves Gram matrices.

\begin{restatable}{lemma}{partialgradientspectralnormbound}[Partial Derivative Gram Matrix Difference Bound]
\label{lemma:partial_gradient_spectral_norm_bound}
For any $l \in \{1, \dots, p\}$,

\begin{align*}
\left\| \frac{\partial \mathbf{K}}{\partial \theta_l} - \frac{\partial \tilde{\mathbf{K}}}{\partial \theta_l} \right\|_2 &\leq \gamma'_{n,m,L,l} \\
&= O\left(\frac{nc^{2d}}{m^{3/d}}\right)
\end{align*}

where $\gamma'_{n,m,L,l}$ is the bound on the spectral norm difference between the kernel matrices corresponding to $k'_{\theta_l}$ and its SKI approximation $\tilde{k}'_{\theta_l}$ (analogous to Proposition \ref{prop:spectral-norm}, but for the kernel $k'_{\theta_l}$).
\end{restatable}
\begin{proof}
See Section \ref{section:proof_partial_gradient_spectral_norm_bound}.    
\end{proof}





We now bound the SKI score function. The key insight to the proof is that the partial derivatives of the difference between regularized gram matrix inverses is in fact a difference between two quadratic forms. We can then use standard techniques \citep{horn2012matrix} for bounding the difference between quadratic forms to obtain our result. The result says that, aside from the response vector's norm, the error grows quadratically in the sample size, at a square root rate in the number of hyperparameters and exponentially in the dimensionality. It further decays at an $m^{\frac{3}{d}}$ rate in the number of inducing points. Noting that to maintain linear time, $m$ should grow at an $n^{d/3}$ rate, we have that aside from the response vector, the error in fact grows linearly with the sample size when choosing the number of inducing points based on Theorem \ref{thm:inducing-points-count-alt}.
\begin{restatable}{lemma}{scorefunctionbound}[Score Function Bound]\label{lemma:score-function-bound}
Let $\mathcal{L}(\boldsymbol{\theta})$ be the true log-likelihood and $\tilde{\mathcal{L}}(\boldsymbol{\theta})$ be the SKI approximation of the log-likelihood at $\boldsymbol{\theta}$. Let $\nabla \mathcal{L}(\boldsymbol{\theta})$ and $\nabla \tilde{\mathcal{L}}(\boldsymbol{\theta})$ denote their respective gradients with respect to $\boldsymbol{\theta}$. Then, for any $\boldsymbol{\theta}\in \mathcal{D}$,

\begin{align*}
&\| \nabla \mathcal{L}(\boldsymbol{\theta}) - \nabla \tilde{\mathcal{L}}(\boldsymbol{\theta}) \|_2 \\
&\leq \frac{1}{2\sigma^4}\Vert \textbf{y}\Vert\sqrt{p}\max_{1\leq l\leq p} \left( \gamma'_{n,m,L,l}+Cn\gamma_{n,m,L}\right.\\
&\qquad\left.+\gamma_{n,m,L}\gamma'_{n,m,L,l} \right)+\frac{\gamma_{n,m,L}}{2\sigma^4}\\
&=\Vert \textbf{y}\Vert_2 O\left(\frac{\sqrt{p}n^2c^{4d}}{m^{3/d}}\right)\\
&\equiv \epsilon_G
\end{align*}
where $C$ is a constants depending on the upper bound of the derivatives of the kernel function over $\mathcal{D}$.
\end{restatable}


\begin{proof}
See Section \ref{sec:proof-score-function-bound}.
\end{proof}

We apply \cite{stonyakin2023stopping} below: the result is the same as in their paper (and assumes $\mu$-smoothness as we did on $\mathcal{L}$), but using gradient ascent instead of descent and using the score function error above. It says that at an $O\left(\frac{1}{K}\right)$ rate, at least one iterate of gradient ascent has its squared gradient norm approach a neighborhood proportional to the squared SKI score function's spectral norm error.

\begin{theorem} \citep{stonyakin2023stopping}
    For inexact gradient ascent on $\mathcal{L}$ with additively inexact gradients satisfying $\|\nabla \mathcal{L}(\boldsymbol{\theta}) - \nabla \tilde{\mathcal{L}}(\boldsymbol{\theta})\| \leq \epsilon_g$, we have:

\begin{equation}
    \max_{k=0,...,N-1} \|\nabla \mathcal{L}(\theta_k)\|^2 \leq \frac{2\mu(\mathcal{L}^* - \mathcal{L}(\boldsymbol{\theta}_0))}{K} + \frac{\epsilon_g^2}{2\mu}
\end{equation}

where $\mathcal{L}^*$ is the value at a stationary point, $\mathcal{L}(\boldsymbol{\theta}_0)$ is the initial, function value, $K$ is the number of iterations and $\epsilon_g$ is the gradient error bound in the previous Lemma.

\end{theorem}

\subsection{Posterior Inference}\label{sec:posterior-inference}
Finally, we treat posterior inference. As the current hyperparameter optimization results only say that \textit{some} iterate approaches a stationary point, we will focus on the error when the SKI and true kernel hyperparameter match.  We first add an assumption
\begin{assumption}
    (Bounded Kernel) Assume that the true kernel satisfies the condition that $|k(\mathbf{x}, \mathbf{x}')| \leq M$ for all $\mathbf{x}, \mathbf{x}'\in \mathcal{X}$.
\end{assumption}


Now we bound the spectral error for the SKI mean function evaluated at a set of test points. The proof follows a standard strategy commonly used for approximate kernel ridge regression. See \cite{bach2013sharp,musco2017recursive} for examples. The result says that the $l^2$ error (aside from the response vector) grows exponentially in the dimensionality, super-linearly but sub-quadratically in the training sample size and at worst linearly in the test sample size. It decays at an $m^{\frac{3}{d}}$ rate in the number of inducing points. Similarly to for the score function error, if we follow Theorem \ref{thm:inducing-points-count-alt} for selecting the number of inducing points, the error in fact grows \textit{sublinearly} with the training sample size.
\begin{restatable}{lemma}{meaninference}\label{lemma:mean-inference} (SKI Posterior Mean Error)
    Let $\boldsymbol{\mu}(\cdot)$ be the GP posterior mean at a set of test points $\cdot\in \mathbb{R}^{T\times d}$ and $\tilde{\boldsymbol{\mu}}(\cdot)$ be the SKI posterior mean at those points. Then the SKI posterior mean $l^2$ error is bounded by:
{\footnotesize
\begin{align*}
    &\Vert \tilde{\boldsymbol{\mu}}(\cdot)- \boldsymbol{\mu}(\cdot)\Vert_2\\
    &\leq\left(\frac{\max(\gamma_{T,m,L},\gamma_{n,m,L})}{\sigma^2}+\frac{\sqrt{Tn}Mc^{2d}}{\sigma^4}\gamma_{n,m,L}\right)\Vert \textbf{y}\Vert_2\\
    &=\Vert \textbf{y}\Vert_2O\left(c^{2d}\frac{\max(T,n)+\sqrt{Tn}n}{m^{3/d}}\right)
\end{align*}
}
\end{restatable}

\begin{proof}
See Appendix \ref{sec:proof-mean-inference}.
\end{proof}

We now derive the spectral error bound for the test SKI covariance matrix. The proof involves noticing that a key term is a difference between two quadratic forms, and using standard techniques for bounding such a difference. The result shows that the error grows at worst super-linearly but subquadratically in the number of test points, quadratically in the training sample size and exponentially in the dimension. Interestingly, due to the use of standard techniques for bounding the difference between quadratic forms, the error is only guaranteed to decay with the number of inducing points at an $m^{3/d-1}$ rate, so that it is only guaranteed to decay at all if $d<3$. If we select the number of inducing points to be proportional to $n^{d/3}$, then the error grows at rate $n^{1+d/3}$ for $d<3$. An interesting question is whether alternate techniques can improve the result for higher dimensional settings e.g. $d\geq 3$.

\begin{restatable}{lemma}{skiposteriorcovarianceerror}[SKI Posterior Covariance Error]\label{lemma:ski-posterior-covariance-error}
Let $\boldsymbol{\Sigma}(\cdot)$ be the GP posterior covariance matrix at a set of test points $\cdot\in \mathbb{R}^{T\times d}$ and $\tilde{\boldsymbol{\Sigma}}(\cdot)$ be its SKI approximation. Then
\begin{align*}
    &\Vert \boldsymbol{\Sigma}(\cdot)-\tilde{\boldsymbol{\Sigma}}(\cdot)\Vert_2\\ &\leq \gamma_{T,m,L} + \frac{\sqrt{Tn}M}{\sigma^2} \max(\gamma_{T,m,L},\gamma_{n,m,L})\\
    &\quad+ \frac{\gamma_{n,m,L}}{\sigma^4}Tn m c^{2d} M^2 \\
    &\quad+ \frac{\sqrt{Tn} m c^{2d} M}{\sigma^2} \max(\gamma_{T,m,L},\gamma_{n,m,L}).\\
    &=O\left(\frac{Tn^2mc^{4d}+\sqrt{Tn}mc^{4d}\max(T,n)}{m^{3/d}}\right).
\end{align*}
where $\gamma_{T,m,L}$ is defined as in Proposition \ref{prop:spectral-norm}.
\end{restatable}

\begin{proof}
See Appendix \ref{sec:proof-ski-posterior-covariance-error}

\end{proof}

%% file: body/06discussion.tex
In this paper, we provided the first rigorous theoretical analysis for structured kernel interpolation. A key practical takeaway is that to control the SKI Gram matrix's spectral norm error, the number of inducing points should grow as $n^{d/3}$. Additionally, we showed the spectral norm error of the SKI gram and cross-kernel matrices, and how this impacts achieving a specific error in linear time. We then analyzed kernel hyperparameter estimation, showing that gradient ascent has an iterate approach a ball around a stationary point, where the ball's radius depends on the spectral error of the SKI score function. We concluded with analysis of the error of the SKI posterior mean and variance.

This work could be extended by analyzing the error of SKI with other interpolation schemes such as Lagrange interpolation \citep{lagrange1795lecons}, using potentially higher order polynomials. This would allow us to not only analyze how to vary $m$ for fixed $L=4$, but how to vary them jointly. Additionally, one could analyze the error of SKI in more complex settings, such as when the inducing points are not placed on a regular grid \citep{snelson2006sparse} or for non-stationary kernel functions, in which case the computational complexity would no longer be $O(n+m\log m)$. Further, we analyze the optimization properties under gradient ascent: it would be interesting to analyze it under stochastic gradient ascent, analogous to \cite{lin2024stochastic}, but now using inexact noisy SKI gradients. Finally, one could analyze the methods for extending SKI to higher dimensions \cite{kapoor2021skiing,yadav2022kernel} and for faster SKI inference \cite{yadav2021faster}.

This paper heavily used LLMs: particularly reasoning models. The paper idea and early error analysis of the kernel error and the spectral norm error came from the authors. Beyond that LLMs were used to outline the statements to be made, turn initial rough descriptions into more formal language, and attempt to prove the results. In general, LLM attempts at proofs were \textit{wrong}, but could drive insights into a working proof strategy. We also used the versions with internet access to help bring up relevant papers. While the LLMs sometimes hallucinated papers, the rate at which it did so was quite low and the usefulness of the papers it found was often very high.

%% file: body/07broader-impact.tex
This work contributes to a deeper theoretical understanding Structured Kernel Interpolation (SKI) \citep{wilson2015kernel} for Gaussian Processes (GPs). By establishing error bounds and analyzing the impact of SKI on hyperparameter estimation and posterior inference, this research can lead to more confident use of approximate Gaussian Processes. These models have broad applications in various domains, including those mentioned in the introduction as well as robotics \citep{deisenroth2015gaussian}, environmental modeling \citep{desai2023deep}, and healthcar \citep{alaa2017bayesian}. Improved Gaussian Process models can enhance prediction accuracy and decision-making, potentially leading to advancements in robotics, more accurate environmental predictions, and better healthcare outcomes. It is important to acknowledge that the application of Gaussian Process models also carries potential risks. For instance, in healthcare, inaccurate predictions or biased models can lead to misdiagnosis or inappropriate treatment \citep{morley2020ethics}. Therefore, understanding potential sources of error when using approximations can be crucial to understanding how reliable we can expect them to be.

%% file: body/appendices/01auxiliary.tex
\begin{lemma}\label{lemma:switch-sum-product}
Given a function $f: \mathbb{R}^d \rightarrow \mathbb{R}$ of the form $f(x_1, x_2, ..., x_d) = \prod_{j=1}^d f_j(x_j)$, where each $f_j: \mathbb{R} \rightarrow \mathbb{R}$. Let $G = G^{(1)} \times G^{(2)} \times ... \times G^{(d)}$ be a fixed d-dimensional grid, where each $G^{(j)} = \{p_1^{(j)}, p_2^{(j)}, ..., p_{n_j}^{(j)}\}$ is a finite set of $n_j$ grid points along the j-th dimension for $j = 1, 2, ..., d$. Then the following equality holds:

$$
\sum_{k_1=1}^{n_1} \sum_{k_2=1}^{n_2} ... \sum_{k_d=1}^{n_d} \prod_{j=1}^d f_j(p_{k_j}^{(j)}) = \prod_{j=1}^d \left( \sum_{k_j=1}^{n_j} f_j(p_{k_j}^{(j)}) \right)
$$
\end{lemma}
\begin{proof}
\textbf{By Induction on d (the number of dimensions):}

\textbf{Base Case (d = 1):}

When $d=1$, the statement becomes:

$$
\sum_{k_1=1}^{n_1} f_1(p_{k_1}^{(1)}) = \sum_{k_1=1}^{n_1} f_1(p_{k_1}^{(1)})
$$

This is trivially true.

\textbf{Inductive Hypothesis:}

Assume the statement holds for $d = m$, i.e.,

$$
\sum_{k_1=1}^{n_1} \sum_{k_2=1}^{n_2} ... \sum_{k_m=1}^{n_m} \prod_{j=1}^m f_j(p_{k_j}^{(j)}) = \prod_{j=1}^m \left( \sum_{k_j=1}^{n_j} f_j(p_{k_j}^{(j)}) \right)
$$

\textbf{Inductive Step:}

We need to show that the statement holds for $d = m+1$. Consider the left-hand side for $d = m+1$:

$$
\sum_{k_1=1}^{n_1} \sum_{k_2=1}^{n_2} ... \sum_{k_{m+1}=1}^{n_{m+1}} \prod_{j=1}^{m+1} f_j(p_{k_j}^{(j)})
$$

We can rewrite this as:

$$
\sum_{k_1=1}^{n_1} \sum_{k_2=1}^{n_2} ... \sum_{k_m=1}^{n_m} \left( \sum_{k_{m+1}=1}^{n_{m+1}} \left( \prod_{j=1}^m f_j(p_{k_j}^{(j)}) \right) f_{m+1}(p_{k_{m+1}}^{(m+1)}) \right)
$$

Notice that the inner sum (over $k_{m+1}$) does not depend on $k_1, k_2, ..., k_m$. Thus, for any fixed values of $k_1, k_2, ..., k_m$, we can treat $\prod_{j=1}^m f_j(p_{k_j}^{(j)})$ as a constant. Let $C(k_1, ..., k_m) = \prod_{j=1}^m f_j(p_{k_j}^{(j)})$. Then we have:

$$
\sum_{k_1=1}^{n_1} \sum_{k_2=1}^{n_2} ... \sum_{k_m=1}^{n_m}  \left( C(k_1, ..., k_m) \sum_{k_{m+1}=1}^{n_{m+1}} f_{m+1}(p_{k_{m+1}}^{(m+1)}) \right)
$$

Now, the inner sum $\sum_{k_{m+1}=1}^{n_{m+1}} f_{m+1}(p_{k_{m+1}}^{(m+1)})$ is a constant with respect to $k_1, ..., k_m$. Let's call this constant $S_{m+1}$. So we have:

$$
\sum_{k_1=1}^{n_1} \sum_{k_2=1}^{n_2} ... \sum_{k_m=1}^{n_m} C(k_1, ..., k_m) S_{m+1} = S_{m+1} \sum_{k_1=1}^{n_1} \sum_{k_2=1}^{n_2} ... \sum_{k_m=1}^{n_m}  \prod_{j=1}^m f_j(p_{k_j}^{(j)})
$$

By the inductive hypothesis, we can replace the nested sums with a product:

$$
S_{m+1} \prod_{j=1}^m \left( \sum_{k_j=1}^{n_j} f_j(p_{k_j}^{(j)}) \right) = \left( \sum_{k_{m+1}=1}^{n_{m+1}} f_{m+1}(p_{k_{m+1}}^{(m+1)}) \right) \prod_{j=1}^m \left( \sum_{k_j=1}^{n_j} f_j(p_{k_j}^{(j)}) \right)
$$

Rearranging the terms, we get:

$$
\prod_{j=1}^m \left( \sum_{k_j=1}^{n_j} f_j(p_{k_j}^{(j)}) \right) \left( \sum_{k_{m+1}=1}^{n_{m+1}} f_{m+1}(p_{k_{m+1}}^{(m+1)}) \right) = \prod_{j=1}^{m+1} \left( \sum_{k_j=1}^{n_j} f_j(p_{k_j}^{(j)}) \right)
$$

This is the right-hand side of the statement for $d = m+1$. Thus, the statement holds for $d = m+1$.

\textbf{Conclusion:}

By induction, the statement holds for all $d \geq 1$. Therefore,

$$
\sum_{k_1=1}^{n_1} \sum_{k_2=1}^{n_2} ... \sum_{k_d=1}^{n_d} \prod_{j=1}^d f_j(p_{k_j}^{(j)}) = \prod_{j=1}^d \left( \sum_{k_j=1}^{n_j} f_j(p_{k_j}^{(j)}) \right)
$$

\end{proof}

\begin{claim}\label{claim:convex-combo-eigenvalues}
Given a convex combination \(\mathbf{C} = \alpha \mathbf{A} + (1-\alpha) \mathbf{B}\), where \(\alpha \in [0,1]\), and \(\mathbf{A}\) and \(\mathbf{B}\) are symmetric matrices, the eigenvalues of \(\mathbf{C}\) lie in the interval \(\left[\min \left(\lambda_n(\mathbf{A}), \lambda_n(\mathbf{B})\right), \max \left(\lambda_1(\mathbf{A}), \lambda_1(\mathbf{B})\right)\right]\).
\end{claim}
\begin{proof}

First, recall that for a symmetric matrix \(\mathbf{A}\), the Rayleigh quotient \(R(\mathbf{A}, \mathbf{x}) = \frac{\mathbf{x}^{\top} \mathbf{A} \mathbf{x}}{\mathbf{x}^{\top} \mathbf{x}}\) is bounded by the smallest and largest eigenvalues of \(\mathbf{A}\):
\[
\lambda_n(\mathbf{A}) \leq R(\mathbf{A}, \mathbf{x}) \leq \lambda_1(\mathbf{A})
\]

Consider the Rayleigh quotient for the matrix \(\mathbf{C}\):
\[
R(\mathbf{C}, \mathbf{x}) = \frac{\mathbf{x}^{\top} (\alpha \mathbf{A} + (1-\alpha) \mathbf{B}) \mathbf{x}}{\mathbf{x}^{\top} \mathbf{x}} = \alpha R(\mathbf{A}, \mathbf{x}) + (1-\alpha) R(\mathbf{B}, \mathbf{x})
\]

Since \(R(\mathbf{A}, \mathbf{x})\) and \(R(\mathbf{B}, \mathbf{x})\) are bounded by their respective eigenvalues, we have:
\[
R(\mathbf{C}, \mathbf{x}) \leq \alpha \lambda_1(\mathbf{A}) + (1-\alpha) \lambda_1(\mathbf{B})
\]
which implies:
\[ 
R(\mathbf{C}, x)  \leq   \max(\lambda_1(\textbf{A}), \lambda_1(\textbf{B}))
\]

Similarly,
\[ 
R(\textbf{C}, \textbf{x})  \geq  \min(\lambda_n(\textbf{A}), \lambda_n(\textbf{B}))
\]

Thus, the eigenvalues of \(\textbf{C} = \alpha \textbf{A} + (1-\alpha)\textbf{B}\) are bounded by:
\[ 
\min(\lambda_n(\textbf{A}), \lambda_n(\textbf{B})) \leq  \lambda(\textbf{C}) \leq \max(\lambda_1(\textbf{A}), \lambda_1(\textbf{B}))
\]
\end{proof}

%% file: body/appendices/04important-quantities.tex
\subsection{Proofs Related to Ski Kernel Error Bounds}\label{subsec:proofs-ski-kernel-error}
\input{body/appendices/important-quantities/polynomial-interpolation}

\input{body/appendices/important-quantities/error-bounds-ski-kernel}

\input{body/appendices/important-quantities/spectral-norm-error}
\subsection{Proofs Related to Linear Time Analysis}
\input{body/appendices/important-quantities/linear-time}

%% file: body/appendices/important-quantities/polynomial-interpolation.tex
\subsubsection{Proof of Lemma \ref{lemma:tensor-product-interpolation-error}}\label{sec:proof-tensor-product-interpolation}
\tensorproductinterpolationerror*
\begin{proof}
We define a sequence of intermediate interpolation functions. Let $g_0(\mathbf{x}) \equiv f(\mathbf{x})$ be the original function. For $i = 1, \ldots, d$, we recursively define $g_i(\mathbf{x})$ as the function obtained by interpolating $g_{i-1}$ along the $i$-th dimension using the cubic convolution kernel $u$:

$$
g_i(\mathbf{x}) \equiv \sum_{k=-1}^2 g_{i-1}\left(\mathbf{x} + \left( (\mathbf{c_x})_i - x_i + kh\right)\mathbf{e}_i \right) u\left(\frac{x_i - (\mathbf{c_x})_i - kh}{h}\right).
$$

Here, $\mathbf{c_x}$ is the grid point closest to $\mathbf{x}$, and $\mathbf{e}_i$ is the $i$-th standard basis vector. Thus, $g_1(\mathbf{x})$ interpolates $f$ along the first dimension, $g_2(\mathbf{x})$ interpolates $g_1$ along the second dimension, and so on, until $g_d(\mathbf{x}) = g(\mathbf{x})$ is the final tensor-product interpolated function.

We analyze the error accumulation across multiple dimensions using induction. Using \cite{keys1981cubic}, the error introduced by interpolating a thrice continuous differentiable function along a single dimension with the cubic convolution kernel is uniformly bounded over the interval domain by $Kh^3$ for some constant $K > 0$, provided the grid spacing $h$ is sufficiently small. This gives us the base case:
$$
|g_1(\mathbf{x}) - g_0(\mathbf{x})| \leq Kh^3.
$$

For the inductive step, assume that for some $i=k$ the error is uniformly bounded by
$$
|g_k(\mathbf{x}) - g_{k-1}(\mathbf{x})| \leq c^{k-1}Kh^3.
$$

We want to show that this bound also holds for $i=k+1$. We can express the difference $g_{k+1}(\mathbf{x}) - g_k(\mathbf{x})$ as follows:

\begin{align*}
g_{k+1}(\mathbf{x}) - g_k(\mathbf{x}) &= \sum_{k_{k+1}=-1}^2 g_k\left(\mathbf{x} + ((\mathbf{c_x})_{k+1} - x_{k+1} + k_{k+1}h) \mathbf{e}_{k+1} \right) u\left(\frac{x_{k+1} - (\mathbf{c_x})_{k+1} - k_{k+1}h}{h}\right) \\
&\quad - g_k(\mathbf{x}) \\
&= \sum_{k_{k+1}=-1}^2 \left[\sum_{k_k=-1}^2 g_{k-1}\left(\mathbf{x} + ((\mathbf{c_x})_k - x_k + k_k h) \mathbf{e}_k + ((\mathbf{c_x})_{k+1} - x_{k+1} + k_{k+1}h) \mathbf{e}_{k+1} \right) \right. \\
&\quad \left. u\left(\frac{x_k - (\mathbf{c_x})_k - k_k h}{h}\right)\right] u\left(\frac{x_{k+1} - (\mathbf{c_x})_{k+1} - k_{k+1}h}{h}\right) \\
&\quad - \sum_{k_k=-1}^2 g_{k-1}\left(\mathbf{x} + ((\mathbf{c_x})_k - x_k + k_k h) \mathbf{e}_k \right) u\left(\frac{x_k - (\mathbf{c_x})_k - k_k h}{h}\right) \\
&= \sum_{k_k=-1}^2 u\left(\frac{x_k - (\mathbf{c_x})_k - k_k h}{h}\right) \left[\sum_{k_{k+1}=-1}^2 g_{k-1}\left(\mathbf{x} + ((\mathbf{c_x})_k - x_k + k_k h) \mathbf{e}_k + ((\mathbf{c_x})_{k+1} - x_{k+1} + k_{k+1}h) \mathbf{e}_{k+1} \right) \right. \\
&\quad \left. u\left(\frac{x_{k+1} - (\mathbf{c_x})_{k+1} - k_{k+1}h}{h}\right) - g_{k-1}\left(\mathbf{x} + ((\mathbf{c_x})_k - x_k + k_k h) \mathbf{e}_k\right)\right].
\end{align*}

The inner term in the last expression represents the difference between interpolating $g_{k-1}$ along the $(k+1)$-th dimension and $g_{k-1}$ itself, evaluated at $\mathbf{x} + ((\mathbf{c_x})_k - x_k + k_k h) \mathbf{e}_k$. This can be written as:

\begin{align*}
&\sum_{k_{k+1}=-1}^2 g_{k-1}\left(\mathbf{x} + ((\mathbf{c_x})_k - x_k + k_k h) \mathbf{e}_k + ((\mathbf{c_x})_{k+1} - x_{k+1} + k_{k+1}h) \mathbf{e}_{k+1} \right) u\left(\frac{x_{k+1} - (\mathbf{c_x})_{k+1} - k_{k+1}h}{h}\right) \\
&\quad - g_{k-1}\left(\mathbf{x} + ((\mathbf{c_x})_k - x_k + k_k h) \mathbf{e}_k\right) \\
&= g_k\left(\mathbf{x} + ((\mathbf{c_x})_k - x_k + k_k h) \mathbf{e}_k \right) - g_{k-1}\left(\mathbf{x} + ((\mathbf{c_x})_k - x_k + k_k h) \mathbf{e}_k \right).
\end{align*}

Therefore, we can bound the error as follows:

\begin{align*}
|g_{k+1}(\mathbf{x}) - g_k(\mathbf{x})| &\leq \left|\sum_{k_k=-1}^2 u\left(\frac{x_k - (\mathbf{c_x})_k - k_k h}{h}\right)\right| \cdot \left|g_k\left(\mathbf{x} + ((\mathbf{c_x})_k - x_k + k_k h) \mathbf{e}_k \right) - g_{k-1}\left(\mathbf{x} + ((\mathbf{c_x})_k - x_k + k_k h) \mathbf{e}_k \right)\right|.
\end{align*}

Let $c>0$ be a uniform upper bound for $\sum_{k_k=-1}^2 \left|u\left(\frac{x_k - (\mathbf{c_x})_k - k_k h}{h}\right)\right|$, which exists because $u$ is bounded. By the inductive hypothesis, we have $\left|g_k\left(\mathbf{x} + ((\mathbf{c_x})_k - x_k + k_k h) \mathbf{e}_k \right) - g_{k-1}\left(\mathbf{x} + ((\mathbf{c_x})_k - x_k + k_k h) \mathbf{e}_k \right)\right| \leq c^{k-1}Kh^3$. Thus,

$$
|g_{k+1}(\mathbf{x}) - g_k(\mathbf{x})| \leq c \cdot c^{k-1}Kh^3 = c^k Kh^3.
$$

This completes the inductive step.

Finally, we bound the total error $|g(\mathbf{x}) - f(\mathbf{x})| = |g_d(\mathbf{x}) - g_0(\mathbf{x})|$ by summing the errors introduced at each interpolation step:

$$
|g(\mathbf{x}) - f(\mathbf{x})| \leq \sum_{i=1}^d |g_i(\mathbf{x}) - g_{i-1}(\mathbf{x})| \leq \sum_{i=1}^d c^{i-1}Kh^3 = Kh^3 \sum_{i=0}^{d-1} c^i.
$$

The last sum is a geometric series, which evaluates to $Kh^3 \frac{1 - c^d}{1 - c}$. For a fixed $c>1$ (independent of $d$), this expression is $O(c^{d})$ when $d$ is large. Therefore, tensor-product cubic convolutional interpolation has $O(c^d h^3)$ error. Finally, noticing that $h=O\left(\frac{1}{m^{d/3}}\right)$ gives us the desired result.
\end{proof}

\subsubsection{Curse of Dimensionality for Kernel Regression}

The next lemma shows that when using a product kernel for $d$-dimensional kernel regression (where cubic convolutional interpolation is a special case), the sum of weights suffers from the curse of dimensionality. The proof strategy involves expressing the multi-dimensional sum as a product of sums over each individual dimension, leveraging the initial condition on the one-dimensional bound for each dimension, and taking advantage of the structure of the Cartesian grid.

\begin{restatable}{lemma}{cubicInterpolationWeights}
\label{lemma:cubic-interpolation-weights-curse-dimensionality}
Let $u: \mathbb{R} \rightarrow \mathbb{R}$ be a one-dimensional kernel function with constant $c>0$ defined as in \ref{def:sum-weight-upper-bound}. Let $u_d: \mathbb{R}^d \rightarrow \mathbb{R}$ be a d-dimensional product kernel defined as:
$$
u_d\left(\frac{x - x_i}{h}\right) = \prod_{j=1}^d u\left(\frac{x^{(j)} - x_i^{(j)}}{h}\right),
$$
where $x = (x^{(1)}, x^{(2)}, ..., x^{(d)}) \in \mathbb{R}^d$ and $x_i = (x_i^{(1)}, x_i^{(2)}, ..., x_i^{(d)}) \in \mathbb{R}^d$ are d-dimensional points. Assume the data points $\{x_i\}_{i=1}^n$ ($n$ may differ from the univariate case) lie on a fixed d-dimensional grid $G = G^{(1)} \times G^{(2)} \times ... \times G^{(d)}$, where each $G^{(j)} = \{p_1^{(j)}, p_2^{(j)}, ..., p_{n_j}^{(j)}\}$ is a finite set of $n_j$ grid points along the j-th dimension for $j = 1, 2, ..., d$. Then, for any $x \in \mathbb{R}^d$, the sum of weights in the d-dimensional kernel regression is bounded by $c^d$:
$$
\sum_{i=1}^n \left|u_d\left(\frac{x - x_i}{h}\right) \right|\leq c^d.
$$
\end{restatable}

\begin{proof}
Let the fixed d-dimensional grid be defined by the Cartesian product of d sets of 1-dimensional grid points: $G = G^{(1)} \times G^{(2)} \times ... \times G^{(d)}$, where $G^{(j)} = \{p_1^{(j)}, p_2^{(j)}, ..., p_{n_j}^{(j)}\}$ is the set of grid points along the j-th dimension.

We start with the sum of weights in the d-dimensional case:

$$
\sum_{i=1}^n u_d\left(\frac{x - x_i}{h}\right) = \sum_{i=1}^n \prod_{j=1}^d u\left(\frac{x^{(j)} - x_i^{(j)}}{h}\right)
$$

Since the data points lie on the fixed grid $G$, we can rewrite the outer sum as a nested sum over the grid points in each dimension:

$$
\sum_{i=1}^n \prod_{j=1}^d u\left(\frac{x^{(j)} - x_i^{(j)}}{h}\right) = \sum_{k_1=1}^{n_1} \sum_{k_2=1}^{n_2} ... \sum_{k_d=1}^{n_d} \prod_{j=1}^d u\left(\frac{x^{(j)} - p_{k_j}^{(j)}}{h}\right)
$$

Now we can change the order of summation and product, as proven in Lemma \ref{lemma:switch-sum-product}:

$$
\sum_{k_1=1}^{n_1} \sum_{k_2=1}^{n_2} ... \sum_{k_d=1}^{n_d} \prod_{j=1}^d u\left(\frac{x^{(j)} - p_{k_j}^{(j)}}{h}\right) = \prod_{j=1}^d \left( \sum_{k_j=1}^{n_j} u\left(\frac{x^{(j)} - p_{k_j}^{(j)}}{h}\right) \right)
$$

By the assumption of the lemma, we know that for each dimension $j$, the sum of weights is bounded by $c$. Note that $\{p_{k_j}^{(j)}\}_{k_j=1}^{n_j}$ is simply a set of points in $\mathbb{R}$, thus:

$$
\sum_{k_j=1}^{n_j} \left|u\left(\frac{x^{(j)} - p_{k_j}^{(j)}}{h}\right)\right| \leq c
$$

Therefore, we have:

$$
\prod_{j=1}^d \left(\left| \sum_{k_j=1}^{n_j} u\left(\frac{x^{(j)} - p_{k_j}^{(j)}}{h}\right)\right| \right) \leq \prod_{j=1}^d c = c^d
$$

Thus, we have shown that:

$$
\sum_{i=1}^n \left|u_d\left(\frac{x - x_i}{h}\right) \right|\leq c^d
$$

\end{proof}

%% file: body/appendices/important-quantities/error-bounds-ski-kernel.tex
\subsubsection{Proof of Lemma \ref{lemma:ski-kernel-elementwise-error}}\label{sec:proof-ski-kernel-elementwise-error}
\skikernelelementwiseerror*
\begin{proof}
Recall that SKI approximates the kernel as
\begin{align*}
k(\mathbf{x}, \mathbf{x}') &\approx \tilde{k}(\mathbf{x}, \mathbf{x}')\\
&= \boldsymbol{w}(\mathbf{x})^\top \mathbf{K}_{\textbf{U}} \boldsymbol{w}(\mathbf{x}'),
\end{align*}

Let $\textbf{K}_{\textbf{U},\textbf{x}'}\in \mathbb{R}^m$ be the vector of kernels between the inducing points and the vector $\textbf{x}'$
\begin{align}
\vert k(\mathbf{x}, \mathbf{x}') -\tilde{k}(\mathbf{x}, \mathbf{x}')\vert &= \vert k(\mathbf{x}, \mathbf{x}')-\boldsymbol{w}(\mathbf{x})^\top \textbf{K}_{\textbf{U},\textbf{x}'} +\boldsymbol{w}(\mathbf{x})^\top \textbf{K}_{\textbf{U},\textbf{x}'}-\boldsymbol{w}(\mathbf{x})^\top \mathbf{K}_{\textbf{U}} \boldsymbol{w}(\mathbf{x}')\vert\nonumber\\
&\leq  \vert k(\mathbf{x}, \mathbf{x}')-\boldsymbol{w}(\mathbf{x})^\top \textbf{K}_{\textbf{U},\textbf{x}'} \vert+\vert \boldsymbol{w}(\mathbf{x})^\top \textbf{K}_{\textbf{U},\textbf{x}'}-\boldsymbol{w}(\mathbf{x})^\top \mathbf{K}_{\textbf{U}} \boldsymbol{w}(\mathbf{x}')\vert\nonumber\\
&\leq \delta_{m,L}+\vert \boldsymbol{w}(\mathbf{x})^\top \textbf{K}_{\textbf{U},\textbf{x}'}-\boldsymbol{w}(\mathbf{x})^\top \mathbf{K}_{\textbf{U}} \boldsymbol{w}(\mathbf{x}')\vert \textrm{ since $\vert k(\mathbf{x}, \mathbf{x}')-\boldsymbol{w}(\mathbf{x})^\top \textbf{K}_{\textbf{U},\textbf{x}'} \vert$ is a single polynomial interpolation}\label{eqn:applying-single-poly-interp}
\end{align}
Now note that $\textbf{w}(x)\in \mathbb{R}^m$ is a sparse matrix with at most $L$ non-zero entries. Thus, letting $\tilde{\textbf{w}}(x)\in \mathbb{R}^L$ be the non-zero entries of $\textbf{w}(x)$ and similarly $\tilde{\textbf{K}}_{\textbf{U},\textbf{x}'}\in \mathbb{R}^L$ be the entries of $\textbf{K}_{\textbf{U},\textbf{x}'}$ in the dimensions corresponding to non-zero entries of $\textbf{w}(x)\in \mathbb{R}^m$, while $\tilde{\textbf{K}}_{\textbf{U}}\in \mathbb{R}^{L\times m}$ is the analogous matrix for $\textbf{K}_{\textbf{U}}$, we have
\begin{align}
    \vert \boldsymbol{w}(\mathbf{x})^\top \textbf{K}_{\textbf{U},\textbf{x}'}-\boldsymbol{w}(\mathbf{x})^\top \mathbf{K}_{\textbf{U}} \boldsymbol{w}(\mathbf{x}')\vert&=\vert \tilde{\textbf{w}}(\textbf{x})^\top \tilde{\textbf{K}}_{\textbf{U},\textbf{x}'}-\tilde{\textbf{w}}(\textbf{x})^\top \tilde{\textbf{K}}_\textbf{U}\textbf{w}(\textbf{x}')\vert\nonumber\\
    &\leq \Vert \tilde{\textbf{w}}(\textbf{x})\Vert_2\Vert \tilde{\textbf{K}}_{\textbf{U},\textbf{x}'}-\tilde{\mathbf{K}}_{\textbf{U}} \boldsymbol{w}(\mathbf{x}')\Vert_2\nonumber\\
    &\leq c^d\sqrt{L}\Vert \tilde{\textbf{K}}_{\textbf{U},\textbf{x}'}-\tilde{\mathbf{K}}_{\textbf{U}} \boldsymbol{w}(\mathbf{x}')\Vert_\infty\nonumber\text{ Lemma \ref{lemma:cubic-interpolation-weights-curse-dimensionality}}\\
    &\leq \sqrt{L}c^d\delta_{m,L}\label{eqn:applying-sparsity-single-kernel-evaluation}
\end{align}
where the last line follows as each element of $\mathbf{K}_{\textbf{U}} \boldsymbol{w}(\mathbf{x}')$ is a polynomial interpolation approximating each element of $\textbf{K}_{\textbf{U},\textbf{x}'}$. Plugging Eqn. \ref{eqn:applying-sparsity-single-kernel-evaluation} into Eqn. \ref{eqn:applying-single-poly-interp} gives us the desired initial result of
\begin{align*}
    \vert k(\textbf{x},\textbf{x}')-\tilde{k}(\textbf{x},\textbf{x}')\vert&\leq \delta_{m,L}+\sqrt{L}c^d\delta_{m,L}
\end{align*}
and Lemma \ref{lemma:tensor-product-interpolation-error} gives us the result when the convolutional kernel is cubic.
\end{proof}

%% file: body/appendices/important-quantities/spectral-norm-error.tex
\subsubsection{Proof of Proposition \ref{prop:spectral-norm}}\label{sec:proof-spectral-norm}

\spectralnorm*
\begin{proof}
    Recall that for any matrix $\textbf{A}$, $\Vert \textbf{A} \Vert_2 \leq \sqrt{\Vert \textbf{A} \Vert_1 \Vert \textbf{A} \Vert_\infty}$. Since $\textbf{K}-\tilde{\textbf{K}}$ is symmetric, we have
    \begin{align*}
        \Vert \textbf{K}-\tilde{\textbf{K}}\Vert_2&\leq \sqrt{\Vert \textbf{K}-\tilde{\textbf{K}}\Vert_1\Vert \textbf{K}-\tilde{\textbf{K}}\Vert_\infty} = \Vert \textbf{K}-\tilde{\textbf{K}}\Vert_\infty
    \end{align*}
    Furthermore, $\Vert \textbf{K}-\tilde{\textbf{K}}\Vert_\infty$ is the maximum absolute row sum of $\textbf{K}-\tilde{\textbf{K}}$. Since there are $n$ rows and, by Lemma \ref{lemma:ski-kernel-elementwise-error}, each element of $\textbf{K} - \tilde{\textbf{K}}$ is bounded by $\delta_{m,L}+\sqrt{L}c^d\delta_{m,L}$ in absolute value, we have
    \begin{align*}
        \Vert \textbf{K}-\tilde{\textbf{K}}\Vert_\infty &\leq n \left(\delta_{m,L}+\sqrt{L} c^d\delta_{m,L}\right) = \gamma_{n,m,L}.
    \end{align*}
    Therefore, $\Vert \textbf{K}-\tilde{\textbf{K}}\Vert_2 \leq \gamma_{n,m,L}$.
\end{proof}

\subsubsection{Proof of Lemma \ref{lemma:test-train-kernel-matrix-error}}\label{sec:proof-test-train-kernel-matrix-error}
\testtrainkernelmatrixerror*
\begin{proof}
    Using the same reasoning as in Proposition \ref{prop:spectral-norm}, we have
    \begin{align*}
        \Vert \textbf{K}_{\cdot,\textbf{X}}-\tilde{\textbf{K}}_{\cdot,\textbf{X}}\Vert_2&\leq \sqrt{\Vert \textbf{K}_{\cdot,\textbf{X}}-\tilde{\textbf{K}}_{\cdot,\textbf{X}}\Vert_1\Vert \textbf{K}_{\cdot,\textbf{X}}-\tilde{\textbf{K}}_{\cdot,\textbf{X}}\Vert_\infty} \\
        &\leq \max \left(\Vert \textbf{K}_{\cdot,\textbf{X}}-\tilde{\textbf{K}}_{\cdot,\textbf{X}}\Vert_1, \Vert \textbf{K}_{\cdot,\textbf{X}}-\tilde{\textbf{K}}_{\cdot,\textbf{X}}\Vert_\infty\right).
    \end{align*}
    Now, $\Vert \textbf{K}_{\cdot,\textbf{X}}-\tilde{\textbf{K}}_{\cdot,\textbf{X}}\Vert_1$ is the maximum absolute column sum, which is less than or equal to $T(\delta_{m,L} + \sqrt{L}c^d\delta_{m,L}) = \gamma_{T,m,L}$. Similarly, $\Vert \textbf{K}_{\cdot,\textbf{X}}-\tilde{\textbf{K}}_{\cdot,\textbf{X}}\Vert_\infty$ is the maximum absolute row sum, which is upper bounded by $n(\delta_{m,L} + \sqrt{L}c^d\delta_{m,L}) = \gamma_{n,m,L}$. Therefore,
    $$
    \Vert \textbf{K}_{\cdot,\textbf{X}}-\tilde{\textbf{K}}_{\cdot,\textbf{X}}\Vert_2 \leq \max(\gamma_{T,m,L},\gamma_{n,m,L}).
    $$
\end{proof}
\subsubsection{Additional Spectral Norm Bounds}
\begin{restatable}{lemma}{traintestbound}\label{lemma:test-train-bound}
    Let $\textbf{K}_{\cdot,\textbf{X}} \in \mathbb{R}^{T \times n}$ be cross kernel matrix between $T$ test points and $n$ training points, where the SKI approximation uses $m$ inducing points. If the kernel function $k$ is bounded such that $|k(\textbf{x}, \textbf{x}')| \leq M$ for all $\textbf{x}, \textbf{x}'\in \mathcal{X}$, then:
    \begin{align*}
        \Vert \textbf{K}_{\cdot,\textbf{X}}\Vert_2&\leq \sqrt{Tn}M
    \end{align*}
\end{restatable}
\begin{proof}
\begin{align*}
    \Vert \textbf{K}_{\cdot,\textbf{X}}\Vert_2&\leq \sqrt{\Vert \textbf{K}_{\cdot,\textbf{X}}\Vert_1\Vert \textbf{K}_{\cdot,\textbf{X}}\Vert_\infty}\\
    &\leq \sqrt{Tn}M
\end{align*}
\end{proof}
\begin{restatable}{lemma}{skitesttrainbound}\label{lemma:ski-test-train-bound}
    Let $\tilde{\textbf{K}}_{\cdot,\textbf{X}} \in \mathbb{R}^{T \times n}$ be the matrix of SKI kernel evaluations between $T$ test points and $n$ training points, where the SKI approximation uses $m$ inducing points. Let $\textbf{W}(\cdot) \in \mathbb{R}^{T \times m}$ and $\textbf{W}(\textbf{X}) \in \mathbb{R}^{n \times m}$ be the matrices of interpolation weights for the test points and training points, respectively. Assume that the interpolation scheme is such that the sum of the absolute values of the interpolation weights for any point is bounded by $c^d$, where $c>0$ is a constant. Let $\textbf{K}_{\textbf{U}} \in \mathbb{R}^{m \times m}$ be the kernel matrix evaluated at the inducing points. If the kernel function $k$ is bounded such that $|k(\textbf{x}, \textbf{x}')| \leq M$ for all $\textbf{x}, \textbf{x}'\in \mathcal{X}$, then:
    $$
    \Vert \tilde{\textbf{K}}_{\cdot,\textbf{X}}\Vert_2 \leq \sqrt{Tn} m c^{2d} M
    $$
\end{restatable}
\begin{proof}
    By the definition of the SKI approximation and the submultiplicativity of the spectral norm, we have:
    $$
    \Vert \tilde{\textbf{K}}_{\cdot,\textbf{X}}\Vert_2 = \Vert \textbf{W}(\cdot)\textbf{K}_{\textbf{U}}(\textbf{W}(\textbf{X}))^\top\Vert_2 \leq \Vert \textbf{W}(\cdot)\Vert_2 \Vert\textbf{K}_{\textbf{U}}\Vert_2 \Vert\textbf{W}(\textbf{X})\Vert_2
    $$

    We now bound each term.

    1.  **Bounding $\Vert \textbf{W}(\cdot)\Vert_2$ and $\Vert \textbf{W}(\textbf{X})\Vert_2$:**
        Since the spectral norm is induced by the Euclidean norm, and using the assumption that the sum of absolute values of interpolation weights for any point is bounded by $c^d$, we have
        $$\Vert \textbf{W}(\cdot)\Vert_2 \leq \sqrt{\Vert \textbf{W}(\cdot)\Vert_1 \Vert \textbf{W}(\cdot)\Vert_\infty} \leq \sqrt{T c^d \cdot c^d} = \sqrt{T} c^d.$$
        Similarly, $\Vert\textbf{W}(\textbf{X})\Vert_2 \leq \sqrt{n}c^d$.

    2.  **Bounding $\Vert\textbf{K}_{\textbf{U}}\Vert_2$:**
        Since $\textbf{K}_{\textbf{U}}$ is symmetric, $\Vert \textbf{K}_{\textbf{U}} \Vert_2 \leq \Vert \textbf{K}_{\textbf{U}} \Vert_\infty$. Each entry of $\textbf{K}_{\textbf{U}}$ is bounded by $M$ (by the boundedness of $k$), and each row has $m$ entries, so $\Vert \textbf{K}_{\textbf{U}} \Vert_\infty \leq mM$. Thus, $\Vert\textbf{K}_{\textbf{U}}\Vert_2 \leq mM$.

    Combining these bounds, we get:
    $$
    \Vert \tilde{\textbf{K}}_{\cdot,\textbf{X}}\Vert_2 \leq (\sqrt{T} c^d) (mM) (\sqrt{n} c^d) = \sqrt{Tn} m c^{2d} M
    $$
    as required.
\end{proof}

\begin{lemma}\label{lemma:action-inverse-error}
    Let \(\mathbf{\tilde{K}}\) be the SKI approximation of the kernel matrix \(\mathbf{K}\), and let \(\sigma^2\) be the regularization parameter. The spectral error of the regularized inverse can be bounded as follows:
\begin{align*}
\left\|\left(\mathbf{\tilde{K}} + \sigma^2 \mathbf{I}\right)^{-1} - \left(\mathbf{K} + \sigma^2 \mathbf{I}\right)^{-1} \right\|_2 &\leq \frac{\gamma_{n, m, L}}{\sigma^4} 
\end{align*}
\end{lemma} 

\begin{proof}
    Note that
\[
\left(\mathbf{\tilde{K}} + \sigma^2 \mathbf{I}\right)^{-1} - \left(\mathbf{K} + \sigma^2 \mathbf{I}\right)^{-1} = \left(\mathbf{\tilde{K}} + \sigma^2 \mathbf{I}\right)^{-1} (\mathbf{K} - \mathbf{\tilde{K}}) \left(\mathbf{K} + \sigma^2 \mathbf{I}\right)^{-1}
\]

Taking the spectral norm, we have
\[
\begin{aligned}
\left\|\left(\mathbf{\tilde{K}} + \sigma^2 \mathbf{I}\right)^{-1} - \left(\mathbf{K} + \sigma^2 \mathbf{I}\right)^{-1}\right\|_2 &\leq \left\|\left(\mathbf{\tilde{K}} + \sigma^2 \mathbf{I}\right)^{-1}\right\|_2 \|\mathbf{K} - \mathbf{\tilde{K}}\|_2 \left\|\left(\mathbf{K} + \sigma^2 \mathbf{I}\right)^{-1}\right\|_2 \\
&\leq \gamma_{n, m, L}\left\|\left(\mathbf{\tilde{K}} + \sigma^2 \mathbf{I}\right)^{-1}\right\|_2 \left\|\left(\mathbf{K} + \sigma^2 \mathbf{I}\right)^{-1}\right\|_2\textrm{ by Proposition \ref{prop:spectral-norm}}\\
&\leq   \frac{\gamma_{n, m, L}}{\sigma^4}
\end{aligned}
\]

\end{proof}

%% file: body/appendices/important-quantities/linear-time.tex
\subsubsection{Proof of Theorem \ref{thm:inducing-points-count-alt}}\label{sec:proof-inducing-points-count-alt}
\inducingpointscountalt*
\begin{proof}
    We want to choose $m$ such that the spectral norm error $\Vert \textbf{K} - \tilde{\textbf{K}} \Vert_2 \leq \epsilon$. From Proposition \ref{prop:spectral-norm}, we have:
    $$
    \Vert \textbf{K} - \tilde{\textbf{K}} \Vert_2 \leq n(1 + \sqrt{L}c^d) \delta_{m,L}
    $$
    For cubic interpolation ($L=4$), Lemma \ref{lemma:tensor-product-interpolation-error}, combined with the analysis in Lemma \ref{lemma:tensor-product-interpolation-error}, gives us:
    $$
    \delta_{m,L} \leq K' c^{2d} h^3
    $$
    where $K'$ is a constant that depends only on the kernel function (through its derivatives) and the interpolation scheme, but not on $n$, $m$, $h$, or $d$.

    Therefore, a sufficient condition to ensure $\Vert \textbf{K} - \tilde{\textbf{K}} \Vert_2 \leq \epsilon$ is:
    \begin{equation} \label{eq:sufficient_condition_final}
    n(1 + 2c^d) K' c^{2d} h^3 \leq \epsilon
    \end{equation}

    Since the inducing points are placed on a regular grid with spacing $h$ in each dimension, and the domain is $[-D,D]^d$ and assuming that $2D\mod h\equiv 0$, the number of inducing points $m$ satisfies:

    $$
    m = \left(\frac{2D}{h}\right)^d
    $$

    We can rearrange this to get:

    $$
    h = \frac{2D}{m^{1/d}}
    $$
    Substituting this into the sufficient condition \eqref{eq:sufficient_condition_final}, we get:

    $$
    n (1 + 2c^d) K' c^{2d} \left(\frac{2D}{m^{1/d}}\right)^3 \leq \epsilon
    $$

    Rearranging to isolate $m$, we obtain:

    $$
    m^{3/d} \geq \frac{n}{\epsilon} (1 + 2c^d) K' c^{2d} (8D^3)
    $$

    $$
    m \geq \left( \frac{n}{\epsilon} (1 + 2c^d) K' (8 c^{2d} D^3) \right)^{d/3}
    $$
\end{proof}
\subsubsection{Proof of Corollary \ref{cor:linear-time}}\label{proof:cor-linear-time}
\corlineartime*
\begin{proof}

Assume that
$$
\epsilon \geq \frac{(1 + 2c^d) K' 8 c^{2d} D^3}{C^{3/d}} \cdot \frac{n (\log n)^{3/d}}{n^{3/d}}.
$$
Rearranging this we obtain
\begin{align*}
    \left( \frac{n}{\epsilon} (1 + 2c^d) K' (8 c^{2d} D^3) \right)^{d/3}&\leq C\frac{n}{\log n}.\\
    &=O\left(\frac{n}{\log n}\right).
\end{align*}
Now taking 
\begin{align*}
    m &= \left( \frac{n}{\epsilon} (1 + 2c^d) K' (8 c^{2d} D^3) \right)^{d/3}
\end{align*}
we have that $m=O\left(\frac{n}{\log n}\right)$ and by Theorem \ref{thm:inducing-points-count-alt}, $\Vert \textbf{K}-\tilde{\textbf{K}}\Vert_2\leq \epsilon$. Now plugging in $\frac{n}{\log n}$ into $m\log m$ we obtain
\begin{align*}
    O\left(m\log m\right)&=O\left(\frac{n}{\log n}\log \frac{n}{\log n}\right)\\
    &=O\left(\frac{n}{\log n}\log n-\frac{n}{\log n}\log \log n\right)\\
    &=O(n)
\end{align*}
as desired.
\end{proof}


%% file: body/appendices/05gp-applications.tex
\subsection{Proofs Related to Hyperparameter Estimation}\label{proofs-hyperparameter-estimation}

\input{body/appendices/gp-applications/hyperparameter-estimation}

\subsection{Proofs Related to Posterior Inference}\label{proofs-posterior-inference}
\input{body/appendices/gp-applications/posterior-inference}

%% file: body/appendices/gp-applications/hyperparameter-estimation.tex
\subsubsection{Proof of Lemma \ref{lemma:ski_kernel_derivative_error_kernel}}\label{sec:proofski_kernel_derivative_error_kernelz}

\skikernelderivativeerrorkernel*
\begin{proof}
By assumption, $k'_{\theta_i}(x, x') = \frac{\partial k_{\theta}(x, x')}{\partial \theta_i}$ is a valid SPD kernel. The SKI approximation of $k'_{\theta_i}(x, x')$ using the same inducing points and interpolation scheme as $\tilde{k}_{\theta}(x, x')$ is given by $\tilde{k}'_{\theta}(x, x')$. For the kernel $k'_{\theta_i}(x, x')$, we have:

\begin{align*}
\left\vert k'_{\theta_i}(x, x') - \tilde{k}'_{\theta}(x, x') \right\vert \leq \delta_{m,L}',
\end{align*}

where $\delta_{m,L}'$ is the upper bound on the error of the SKI approximation of $k'_{\theta_i}(x, x')$ as defined in Lemma \ref{lemma:ski-kernel-elementwise-error}.

Now, we need to show that $\frac{\partial \tilde{k}_{\theta}(x,x')}{\partial \theta_i} = \tilde{k}'_{\theta}(x, x')$. Recall that the SKI approximation $\tilde{k}_{\theta}(x, x')$ is a linear combination of kernel evaluations at inducing points, with weights that depend on $x$ and $x'$:

\begin{align*}
\tilde{k}_{\theta}(x, x') = \sum_{j=1}^m \sum_{l=1}^m w_{jl}(x, x') k_{\theta}(u_j, u_l)
\end{align*}

where $w_{jl}(x, x')$ are the interpolation weights. Taking the partial derivative with respect to $\theta_i$, we get:
\begin{align*}
\frac{\partial \tilde{k}_{\theta}(x, x')}{\partial \theta_i} &= \sum_{j=1}^m \sum_{l=1}^m w_{jl}(x, x') \frac{\partial k_{\theta}(u_j, u_l)}{\partial \theta_i} \\
&= \sum_{j=1}^m \sum_{l=1}^m w_{jl}(x, x') k'_{\theta_i}(u_j, u_l).
\end{align*}

This is precisely the SKI approximation of the kernel $k'_{\theta_i}(x, x')$ using the same inducing points and weights:
\begin{align*}
\tilde{k}'_{\theta}(x, x') = \sum_{j=1}^m \sum_{l=1}^m w_{jl}(x, x') k'_{\theta_i}(u_j, u_l).
\end{align*}

Therefore, $\frac{\partial \tilde{k}_{\theta}(x,x')}{\partial \theta_i} = \tilde{k}'_{\theta}(x, x')$.

Substituting this into our inequality, we get:
\begin{align*}
\left\vert \frac{\partial k_{\theta}(x,x')}{\partial \theta_i}-\frac{\partial \tilde{k}_{\theta}(x,x')}{\partial \theta_i}\right\vert &= \left\vert k'_{\theta_i}(x, x') - \tilde{k}'_{\theta}(x, x') \right\vert \\
&\leq \delta_{m,L}'+\sqrt{L}c^d\delta_{m,L}'.
\end{align*}

\end{proof}

\subsubsection{Proof of Lemma \ref{lemma:partial_gradient_spectral_norm_bound}}\label{section:proof_partial_gradient_spectral_norm_bound}
\partialgradientspectralnormbound*
\begin{proof}
Let $K'_{\theta,l}$ be the kernel matrix corresponding to the kernel $k'_{\theta,l}(x,x') = \frac{\partial k_{\theta}(x,x')}{\partial \theta_l}$, and let $\tilde{K}'_{\theta,l}$ be the kernel matrix corresponding to its SKI approximation $\tilde{k}'_{\theta,l}(x,x')$.

From Lemma \ref{lemma:ski_kernel_derivative_error_kernel}, we have:

\begin{equation}
\frac{\partial \tilde{k}_{\theta}(x,x')}{\partial \theta_l} = \tilde{k}'_{\theta,l}(x, x')
\end{equation}

Therefore:
\begin{equation}
\frac{\partial K}{\partial \theta_l} - \frac{\partial \tilde{K}}{\partial \theta_l} = K'_{\theta,l} - \tilde{K}'_{\theta,l}
\end{equation}

By Proposition \ref{prop:spectral-norm}, we have a bound on the spectral norm difference between a kernel matrix and its SKI approximation. Let $\gamma'_{n,m,L,l}$ be the corresponding bound for the kernel $k'_{\theta,l}$ and its SKI approximation $\tilde{k}'_{\theta,l}$. Then:

\begin{equation}
\| K'_{\theta,l} - \tilde{K}'_{\theta,l} \|_2 \leq \gamma'_{n,m,L,l}
\end{equation}

Thus,

\begin{align*}
\left\| \frac{\partial K}{\partial \theta_l} - \frac{\partial \tilde{K}}{\partial \theta_l} \right\|_2 = \| K'_{\theta,l} - \tilde{K}'_{\theta,l} \|_2 \leq \gamma'_{n,m,L,l}
\end{align*}

This completes the proof.
\end{proof}

\subsubsection{Proof of Lemma \ref{lemma:score-function-bound}}\label{sec:proof-score-function-bound}
\scorefunctionbound*
\begin{proof}
    We start with the expressions for the gradients:

$$
\nabla \mathcal{L}(\theta) = \nabla \left( -\frac{1}{2} \mathbf{y}^\top (\mathbf{K} + \sigma^2 \mathbf{I})^{-1} \mathbf{y} - \frac{1}{2} \log |\mathbf{K} + \sigma^2 \mathbf{I}| - \frac{n}{2} \log(2\pi) \right).
$$

$$
\nabla \tilde{\mathcal{L}}(\theta) = \nabla \left( -\frac{1}{2} \mathbf{y}^\top (\tilde{\mathbf{K}} + \sigma^2 \mathbf{I})^{-1} \mathbf{y} - \frac{1}{2} \log |\tilde{\mathbf{K}} + \sigma^2 \mathbf{I}| - \frac{n}{2} \log(2\pi) \right).
$$

Thus, the difference is:

\begin{align*}
\| \nabla \mathcal{L}(\theta) - \nabla \tilde{\mathcal{L}}(\theta) \|_2 &= \left\| \nabla \left( -\frac{1}{2} \mathbf{y}^\top (\mathbf{K} + \sigma^2 \mathbf{I})^{-1} \mathbf{y} - \frac{1}{2} \log |\mathbf{K} + \sigma^2 \mathbf{I}| \right) \right. \\
&\quad \left. - \nabla \left( -\frac{1}{2} \mathbf{y}^\top (\tilde{\mathbf{K}} + \sigma^2 \mathbf{I})^{-1} \mathbf{y} - \frac{1}{2} \log |\tilde{\mathbf{K}} + \sigma^2 \mathbf{I}| \right) \right\|_2 \\
&\leq \underbrace{\left\| \nabla \left( \frac{1}{2} \mathbf{y}^\top \left( (\tilde{\mathbf{K}} + \sigma^2 \mathbf{I})^{-1} - (\mathbf{K} + \sigma^2 \mathbf{I})^{-1} \right) \mathbf{y} \right) \right\|_2}_{T_1} \\
&\quad + \underbrace{\left\| \frac{1}{2} \nabla \left( \log |\mathbf{K} + \sigma^2 \mathbf{I}| - \log |\tilde{\mathbf{K}} + \sigma^2 \mathbf{I}| \right) \right\|_2}_{T_2}.
\end{align*}

We will bound $T_1$ and $T_2$ separately.

\textbf{Bounding $T_1$:}

\begin{align*}
    T_1 &= \frac{1}{2} \left\| \nabla_\theta \left( \mathbf{y}^\top \left( (\tilde{\mathbf{K}} + \sigma^2 \mathbf{I})^{-1} - (\mathbf{K} + \sigma^2 \mathbf{I})^{-1} \right) \mathbf{y} \right) \right\|_2\\
    &=\frac{1}{2}\sqrt{\sum_{l=1}^p \left(\frac{\partial}{\partial \theta_l}\mathbf{y}^\top \left( (\tilde{\mathbf{K}} + \sigma^2 \mathbf{I})^{-1} - (\mathbf{K} + \sigma^2 \mathbf{I})^{-1} \right) \mathbf{y}\right)^2}\\
    &\leq \frac{1}{2}\sqrt{p}\max_{1\leq l\leq p}\sqrt{\left(\frac{\partial}{\partial \theta_l}\mathbf{y}^\top \left( (\tilde{\mathbf{K}} + \sigma^2 \mathbf{I})^{-1} - (\mathbf{K} + \sigma^2 \mathbf{I})^{-1} \right) \mathbf{y}\right)^2}\\
    &=\frac{1}{2}\sqrt{p}\max_{1\leq l\leq p} \left\vert \frac{\partial}{\partial \theta_l}\mathbf{y}^\top \left( (\tilde{\mathbf{K}} + \sigma^2 \mathbf{I})^{-1} - (\mathbf{K} + \sigma^2 \mathbf{I})^{-1} \right) \mathbf{y}\right\vert
\end{align*}
We will then bound $\left\vert \frac{\partial}{\partial \theta_l}\mathbf{y}^\top \left( (\tilde{\mathbf{K}} + \sigma^2 \mathbf{I})^{-1} - (\mathbf{K} + \sigma^2 \mathbf{I})^{-1} \right) \mathbf{y}\right\vert$. Using the following equality $\frac{\partial}{\partial \theta_l} \mathbf{X}^{-1} = -\mathbf{X}^{-1} (\frac{\partial \mathbf{X}}{\partial \theta_l}) \mathbf{X}^{-1}$, we can express this derivative as a quadratic form as a difference between two quadratic forms and apply standard techniques for bounding differences between quadratic forms.
\begin{align*}
    &\left\vert \frac{\partial}{\partial \theta_l}\mathbf{y}^\top \left( (\tilde{\mathbf{K}} + \sigma^2 \mathbf{I})^{-1} - (\mathbf{K} + \sigma^2 \mathbf{I})^{-1} \right) \mathbf{y}\right\vert\\
    &\leq\Vert \textbf{y}\Vert_2^2 \left\Vert \frac{\partial}{\partial \theta_l }\left( (\tilde{\mathbf{K}} + \sigma^2 \mathbf{I})^{-1} - (\mathbf{K} + \sigma^2 \mathbf{I})^{-1} \right) \right\Vert_2\text{ CS inequality}\\
    &=\Vert \textbf{y}\Vert_2^2  \left\Vert \frac{\partial}{\partial \theta_l }\left( (\tilde{\mathbf{K}} + \sigma^2 \mathbf{I})^{-1} - (\mathbf{K} + \sigma^2 \mathbf{I})^{-1} \right) \right\Vert_2\\
    &= \Vert \textbf{y}\Vert_2^2 \left\|  - (\tilde{\mathbf{K}} + \sigma^2 \mathbf{I})^{-1} \left(\frac{\partial}{\partial \theta_l} \tilde{\mathbf{K}}\right) (\tilde{\mathbf{K}} + \sigma^2 \mathbf{I})^{-1} + (\mathbf{K} + \sigma^2 \mathbf{I})^{-1} \left(\frac{\partial}{\partial \theta_l} \mathbf{K}\right) (\mathbf{K} + \sigma^2 \mathbf{I})^{-1}  \right\|_2 \\
    &= \Vert \textbf{y}\Vert_2^2 \left\|  - (\tilde{\mathbf{K}} + \sigma^2 \mathbf{I})^{-1} \left(\frac{\partial}{\partial \theta_l} \tilde{\mathbf{K}} - \frac{\partial}{\partial \theta_l} \mathbf{K} + \frac{\partial}{\partial \theta_l} \mathbf{K}\right) (\tilde{\mathbf{K}} + \sigma^2 \mathbf{I})^{-1} \right.\\
    &\quad \left. + (\mathbf{K} + \sigma^2 \mathbf{I})^{-1} \frac{\partial}{\partial \theta_l} \mathbf{K} (\mathbf{K} + \sigma^2 \mathbf{I})^{-1}  \right\|_2 \\
    &= \Vert \textbf{y}\Vert_2^2 \left\|  - (\tilde{\mathbf{K}} + \sigma^2 \mathbf{I})^{-1} \left(\frac{\partial}{\partial \theta_l} \tilde{\mathbf{K}} - \frac{\partial}{\partial \theta_l} \mathbf{K}\right) (\tilde{\mathbf{K}} + \sigma^2 \mathbf{I})^{-1} \right.\\
    &\quad \left. - (\tilde{\mathbf{K}} + \sigma^2 \mathbf{I})^{-1} \left(\frac{\partial}{\partial \theta_l} \mathbf{K}\right) (\tilde{\mathbf{K}} + \sigma^2 \mathbf{I})^{-1}  + (\mathbf{K} + \sigma^2 \mathbf{I})^{-1} \left(\frac{\partial}{\partial \theta_l} \mathbf{K}\right) (\mathbf{K} + \sigma^2 \mathbf{I})^{-1}  \right\|_2 \\
    &\leq\Vert \textbf{y}\Vert_2^2  \left(\underbrace{\left\| (\tilde{\mathbf{K}} + \sigma^2 \mathbf{I})^{-1} \left(\frac{\partial}{\partial \theta_l} \tilde{\mathbf{K}} - \frac{\partial}{\partial \theta_l} \mathbf{K}\right) (\tilde{\mathbf{K}} + \sigma^2 \mathbf{I})^{-1} \right\|_{2}}_{(a)} \right.\\
    &\quad \left.+ \underbrace{\left\| \left((\tilde{\mathbf{K}} + \sigma^2 \mathbf{I})^{-1} - (\mathbf{K} + \sigma^2 \mathbf{I})^{-1}\right)\left(\frac{\partial}{\partial \theta_l} \mathbf{K}\right) (\tilde{\mathbf{K}} + \sigma^2 \mathbf{I})^{-1} \right\|_{2}}_{(b)} \right.\\
    &\quad \left.+ \underbrace{\left\|(\mathbf{K} + \sigma^2 \mathbf{I})^{-1}\left(\frac{\partial}{\partial \theta_l} \mathbf{K}\right)\left((\tilde{\mathbf{K}} + \sigma^2 \mathbf{I})^{-1} - (\mathbf{K} + \sigma^2 \mathbf{I})^{-1}\right) \right\|_{2}}_{(c)}\right).
\end{align*}

We now explicitly bound (a), (b), and (c).
\begin{align*}
(a) &\leq \left\| (\tilde{\mathbf{K}} + \sigma^2 \mathbf{I})^{-1} \right\|_2 \left\| \frac{\partial}{\partial \theta_l} \tilde{\mathbf{K}} - \frac{\partial}{\partial \theta_l} \mathbf{K} \right\|_2 \left\| (\tilde{\mathbf{K}} + \sigma^2 \mathbf{I})^{-1} \right\|_2\\
&\leq \left\| (\tilde{\mathbf{K}} + \sigma^2 \mathbf{I})^{-1} \right\|_2^2 \left\| \frac{\partial}{\partial \theta_l} \tilde{\mathbf{K}} - \frac{\partial}{\partial \theta_l} \mathbf{K} \right\|_2 \\
&\leq \frac{1}{\sigma^4} \left\| \frac{\partial}{\partial \theta_l} \mathbf{K} - \frac{\partial}{\partial \theta_l} \tilde{\mathbf{K}} \right\|_2\\
&\leq  \frac{1}{\sigma^4} \gamma'_{n,m,L,l} \quad \text{ (Using Lemma \ref{lemma:partial_gradient_spectral_norm_bound})}
\end{align*}

\begin{align*}
(b) &\leq \|(\tilde{\mathbf{K}} + \sigma^2 \mathbf{I})^{-1} - (\mathbf{K} + \sigma^2 \mathbf{I})^{-1}\|_2 \left\|\frac{\partial}{\partial \theta_l} \mathbf{K}\right\|_2 \|(\tilde{\mathbf{K}} + \sigma^2 \mathbf{I})^{-1}\|_2 \\
&\leq \frac{1}{\sigma^2} \|(\tilde{\mathbf{K}} + \sigma^2 \mathbf{I})^{-1} - (\mathbf{K} + \sigma^2 \mathbf{I})^{-1}\|_2 \left\|\frac{\partial}{\partial \theta_l} \mathbf{K}\right\|_2 \\
&\leq \frac{\gamma_{n,m,L}}{\sigma^4} \left\|\frac{\partial}{\partial \theta_l} \mathbf{K}\right\|_2 \quad \text{(Using Lemma \ref{lemma:action-inverse-error})}
\end{align*}
Since the kernel is $C^1$ wrt $\theta$ and $\mathcal{D}$ is compact, we can bound the entries of $\frac{\partial}{\partial \theta_l} \mathbf{K}$ uniformly over $\mathcal{D}$ and $l$ with some constant, say $C>0$. Then by Lemma \ref{lemma:test-train-bound}, reusing the training points instead of using the test points,
\begin{align*}
    (b)&\leq \frac{\gamma_{n,m,L}}{\sigma^4} \left\|\frac{\partial}{\partial \theta_l} \mathbf{K}\right\|_2\\
    &\leq Cn\frac{\gamma_{n,m,L}}{\sigma^4}
\end{align*}
and finally 
\begin{align*}
(c) &\leq \|(\mathbf{K} + \sigma^2 \mathbf{I})^{-1}\|_2 \left\|\frac{\partial}{\partial \theta_l} \mathbf{K}\right\|_2 \|(\tilde{\mathbf{K}} + \sigma^2 \mathbf{I})^{-1} - (\mathbf{K} + \sigma^2 \mathbf{I})^{-1}\|_2 \\
&\leq \frac{1}{\sigma^2} \|(\tilde{\mathbf{K}} + \sigma^2 \mathbf{I})^{-1} - (\mathbf{K} + \sigma^2 \mathbf{I})^{-1}\|_2 \left\|\frac{\partial}{\partial \theta_l} \mathbf{K}\right\|_2 \\
&\leq \frac{\gamma_{n,m,L}}{\sigma^4} \left\|\frac{\partial}{\partial \theta_l} \mathbf{K}\right\|_2 \quad \text{(Using Lemma \ref{lemma:action-inverse-error})} \\
&\leq \frac{\gamma_{n,m,L}}{\sigma^4} \gamma'_{n,m,L,l} \quad \text{(Using Lemma \ref{lemma:partial_gradient_spectral_norm_bound} )}
\end{align*}

Combining these, we obtain
\begin{align*}
    T_1 &\leq \frac{1}{2\sigma^4}\Vert \textbf{y}\Vert\sqrt{p}\max_{1\leq l\leq p} \left( \gamma'_{n,m,L,l}+Cn\gamma_{n,m,L}+\gamma_{n,m,L}\gamma'_{n,m,L,l} \right)
\end{align*}


**Bounding $T_2$:**

Using the identity $\nabla \log |\mathbf{X}| = (\mathbf{X}^{-1})^\top$, we have

\begin{align*}
T_2 &= \frac{1}{2} \left\| \nabla_\theta \left( \log |\mathbf{K} + \sigma^2 \mathbf{I}| - \log |\tilde{\mathbf{K}} + \sigma^2 \mathbf{I}| \right) \right\|_2 \\
&= \frac{1}{2} \left\| (\mathbf{K} + \sigma^2 \mathbf{I})^{-1} - (\tilde{\mathbf{K}} + \sigma^2 \mathbf{I})^{-1} \right\|_2.
\end{align*}
We can rewrite the difference as:
$$
(\mathbf{K} + \sigma^2 \mathbf{I})^{-1} - (\tilde{\mathbf{K}} + \sigma^2 \mathbf{I})^{-1} = (\tilde{\mathbf{K}} + \sigma^2 \mathbf{I})^{-1} (\tilde{\mathbf{K}} - \mathbf{K}) (\mathbf{K} + \sigma^2 \mathbf{I})^{-1}
$$
Then
\begin{align*}
T_2 &\leq \frac{1}{2} \| (\tilde{\mathbf{K}} + \sigma^2 \mathbf{I})^{-1} \|_2 \| \tilde{\mathbf{K}} - \mathbf{K} \|_2 \| (\mathbf{K} + \sigma^2 \mathbf{I})^{-1} \|_2 \\
    &\leq \frac{\gamma_{n,m,L}}{2\sigma^4}
\end{align*}

**Combining the Bounds:**

Combining the bounds for $T_1$ and $T_2$, we have

\begin{align*}
\| \nabla \mathcal{L}(\boldsymbol{\theta}) - \nabla \tilde{\mathcal{L}}(\boldsymbol{\theta}) \|_2 &\leq \frac{1}{2\sigma^4}\Vert \textbf{y}\Vert\sqrt{p}\max_{1\leq l\leq p} \left( \gamma'_{n,m,L,l}+Cn\gamma_{n,m,L}+\gamma_{n,m,L}\gamma'_{n,m,L,l} \right)+\frac{\gamma_{n,m,L}}{2\sigma^4}
\end{align*}
\end{proof}

%% file: body/appendices/gp-applications/posterior-inference.tex
\subsubsection{Proof of Lemma \ref{lemma:mean-inference}}\label{sec:proof-mean-inference}
\meaninference*
\begin{proof}
        We start by expressing the difference between the true and SKI posterior means:

\begin{align*}
& \left\|\mathbf{K}_{\cdot,\mathbf{X}} \left( \mathbf{K} + \sigma^{2} \mathbf{I} \right)^{-1} \mathbf{y} - \tilde{\mathbf{K}}_{\cdot, \mathbf{X}} \left( \tilde{\mathbf{K}} + \sigma^{2} \mathbf{I} \right)^{-1} \mathbf{y} \right\|_{2} \\
&= \left\| \left( \tilde{\mathbf{K}}_{\cdot, \mathbf{X}} -\mathbf{K}_{\cdot,\mathbf{X}}\right) \left( \tilde{\mathbf{K}} + \sigma^{2} \mathbf{I} \right)^{-1} \mathbf{y} +\mathbf{K}_{\cdot,\mathbf{X}}\left[ \left( \tilde{\mathbf{K}} + \sigma^{2} \mathbf{I} \right)^{-1} - \left(\mathbf{K} + \sigma^{2} \mathbf{I} \right)^{-1} \right] \mathbf{y} \right\|_{2}\\
\end{align*}

Applying the triangle inequality and submultiplicative property gives:
\begin{align*}
    &\leq \frac{1}{\sigma^2}\Vert \textbf{y}\Vert_2\Vert \tilde{\mathbf{K}}_{\cdot, \mathbf{X}} -\mathbf{K}_{\cdot,\mathbf{X}}\Vert_2+\Vert\mathbf{K}_{\cdot,\mathbf{X}}\Vert_2\left\Vert \left( \tilde{\mathbf{K}} + \sigma^{2} \mathbf{I} \right)^{-1} - \left(\mathbf{K} + \sigma^{2} \mathbf{I} \right)^{-1}\right\Vert_2\Vert \textbf{y}\Vert_2\\
    &\leq \frac{\max\left(\gamma_{T,m,L},\gamma_{n,m,L}\right)}{\sigma^2}\Vert \textbf{y}\Vert_2+\Vert\mathbf{K}_{\cdot,\mathbf{X}}\Vert_2\left\Vert \left( \tilde{\mathbf{K}} + \sigma^{2} \mathbf{I} \right)^{-1} - \left(\mathbf{K} + \sigma^{2} \mathbf{I} \right)^{-1}\right\Vert_2\Vert \textbf{y}\Vert_2\text{ Lemma \ref{lemma:test-train-kernel-matrix-error}}\\
    &\leq \frac{\max\left(\gamma_{T,m,L},\gamma_{n,m,L}\right)}{\sigma^2}\Vert \textbf{y}\Vert_2+\sqrt{Tn}M\left\Vert \left( \tilde{\mathbf{K}} + \sigma^{2} \mathbf{I} \right)^{-1} - \left(\mathbf{K} + \sigma^{2} \mathbf{I} \right)^{-1}\right\Vert_2\Vert \textbf{y}\Vert_2\text{ Lemma \ref{lemma:test-train-bound}}\\
    &\leq \frac{\max\left(\gamma_{T,m,L},\gamma_{n,m,L}\right)}{\sigma^2}\Vert \textbf{y}\Vert_2+\frac{\sqrt{Tn}M}{\sigma^4}\gamma_{n,m,L}\Vert \textbf{y}\Vert_2 \text{ Lemma \ref{lemma:action-inverse-error}}\\
    &=\frac{1}{\sigma^2}\Vert \textbf{y}\Vert_2\left(\max\left(\gamma_{T,m,L},\gamma_{n,m,L}\right)+\frac{\sqrt{Tn}M}{\sigma^4}\gamma_{n,m,L}\right)\\
    &=\frac{1}{\sigma^2}\Vert \textbf{y}\Vert_2O\left(c^{2d}\frac{\max(T,n)+\sqrt{Tn}Mn}{m^{3/d}}\right)
\end{align*}

\end{proof}

\subsubsection{Proof of Lemma \ref{lemma:ski-posterior-covariance-error}}\label{sec:proof-ski-posterior-covariance-error}

\begin{proof}
    

First, note that
\begin{align*}
    \Vert \boldsymbol{\Sigma}(\cdot)-\tilde{\boldsymbol{\Sigma}}(\cdot)\Vert_2 &\leq \Vert \textbf{K}_{\cdot,\cdot}-\tilde{\textbf{K}}_{\cdot,\cdot}\Vert_2\\
    &+ \Vert \textbf{K}_{\cdot,\textbf{X}} (\textbf{K}+\sigma^2 I)^{-1}\textbf{K}_{\textbf{X},\cdot}-\tilde{\textbf{K}}_{\cdot,\textbf{X}} (\tilde{\textbf{K}}+\sigma^2 I)^{-1}\tilde{\textbf{K}}_{\textbf{X},\cdot}\Vert_2 \\
    &\leq \gamma_{T,m,L} + \Vert \textbf{K}_{\cdot,\textbf{X}} (\textbf{K}+\sigma^2 I)^{-1}\textbf{K}_{\textbf{X},\cdot}-\tilde{\textbf{K}}_{\cdot,\textbf{X}} (\tilde{\textbf{K}}+\sigma^2 I)^{-1}\tilde{\textbf{K}}_{\textbf{X},\cdot}\Vert_2,
\end{align*}
where we used Proposition \ref{prop:spectral-norm} and the fact that $\Vert \textbf{K}_{\cdot,\cdot}-\tilde{\textbf{K}}_{\cdot,\cdot}\Vert_2 \leq \gamma_{T,m,L}$.


Now, we bound the second term, which is a different between two quadratic forms:
\begin{align*}
    &\Vert \textbf{K}_{\cdot,\textbf{X}} (\textbf{K}+\sigma^2 I)^{-1}\textbf{K}_{\textbf{X},\cdot}-\tilde{\textbf{K}}_{\cdot,\textbf{X}} (\tilde{\textbf{K}}+\sigma^2 I)^{-1}\tilde{\textbf{K}}_{\textbf{X},\cdot}\Vert_2\\
    &\leq \Vert \textbf{K}_{\cdot,\textbf{X}} (\textbf{K}+\sigma^2 I)^{-1}\textbf{K}_{\textbf{X},\cdot}-\textbf{K}_{\cdot,\textbf{X}} (\textbf{K}+\sigma^2 I)^{-1}\tilde{\textbf{K}}_{\textbf{X},\cdot}\Vert_2 \\
    &+ \Vert\textbf{K}_{\cdot,\textbf{X}} (\textbf{K}+\sigma^2 I)^{-1}\tilde{\textbf{K}}_{\textbf{X},\cdot}-\tilde{\textbf{K}}_{\cdot,\textbf{X}} (\tilde{\textbf{K}}+\sigma^2 I)^{-1}\tilde{\textbf{K}}_{\textbf{X},\cdot}\Vert_2\\
    &\leq \Vert \textbf{K}_{\cdot,\textbf{X}} (\textbf{K}+\sigma^2 I)^{-1}(\textbf{K}_{\textbf{X},\cdot}-\tilde{\textbf{K}}_{\textbf{X},\cdot})\Vert_2 + \Vert (\textbf{K}_{\cdot,\textbf{X}} (\textbf{K}+\sigma^2 I)^{-1}-\tilde{\textbf{K}}_{\cdot,\textbf{X}} (\tilde{\textbf{K}}+\sigma^2 I)^{-1})\tilde{\textbf{K}}_{\textbf{X},\cdot}\Vert_2\\
    &\leq \Vert \textbf{K}_{\cdot,\textbf{X}} \Vert_2 \Vert (\textbf{K}+\sigma^2 I)^{-1} \Vert_2 \Vert \textbf{K}_{\textbf{X},\cdot}-\tilde{\textbf{K}}_{\textbf{X},\cdot} \Vert_2 + \Vert \textbf{K}_{\cdot,\textbf{X}} (\textbf{K}+\sigma^2 I)^{-1}-\tilde{\textbf{K}}_{\cdot,\textbf{X}} (\tilde{\textbf{K}}+\sigma^2 I)^{-1} \Vert_2 \Vert \tilde{\textbf{K}}_{\textbf{X},\cdot} \Vert_2\\
    &\leq \frac{1}{\sigma^2} \Vert \textbf{K}_{\cdot,\textbf{X}}\Vert_2 \Vert \textbf{K}_{\textbf{X},\cdot}-\tilde{\textbf{K}}_{\textbf{X},\cdot}\Vert_2 + \Vert \textbf{K}_{\cdot,\textbf{X}} (\textbf{K}+\sigma^2 I)^{-1}-\tilde{\textbf{K}}_{\cdot,\textbf{X}} (\tilde{\textbf{K}}+\sigma^2 I)^{-1} \Vert_2 \Vert \tilde{\textbf{K}}_{\textbf{X},\cdot} \Vert_2,
\end{align*}
where we used the fact that $(\textbf{K}+\sigma^2 I)^{-1} \preceq \frac{1}{\sigma^2}I$.

Next, we bound the term $\Vert \textbf{K}_{\cdot,\textbf{X}} (\textbf{K}+\sigma^2 I)^{-1}-\tilde{\textbf{K}}_{\cdot,\textbf{X}} (\tilde{\textbf{K}}+\sigma^2 I)^{-1} \Vert_2$:
\begin{align*}
    &\Vert \textbf{K}_{\cdot,\textbf{X}} (\textbf{K}+\sigma^2 I)^{-1}-\tilde{\textbf{K}}_{\cdot,\textbf{X}} (\tilde{\textbf{K}}+\sigma^2 I)^{-1} \Vert_2 \\
    &= \Vert \textbf{K}_{\cdot,\textbf{X}} (\textbf{K}+\sigma^2 I)^{-1} - \textbf{K}_{\cdot,\textbf{X}} (\tilde{\textbf{K}}+\sigma^2 I)^{-1} + \textbf{K}_{\cdot,\textbf{X}} (\tilde{\textbf{K}}+\sigma^2 I)^{-1} - \tilde{\textbf{K}}_{\cdot,\textbf{X}} (\tilde{\textbf{K}}+\sigma^2 I)^{-1} \Vert_2 \\
    &\leq \Vert \textbf{K}_{\cdot,\textbf{X}} (\textbf{K}+\sigma^2 I)^{-1} - \textbf{K}_{\cdot,\textbf{X}} (\tilde{\textbf{K}}+\sigma^2 I)^{-1}\Vert_2 + \Vert \textbf{K}_{\cdot,\textbf{X}} (\tilde{\textbf{K}}+\sigma^2 I)^{-1} - \tilde{\textbf{K}}_{\cdot,\textbf{X}} (\tilde{\textbf{K}}+\sigma^2 I)^{-1} \Vert_2\\
    &= \Vert \textbf{K}_{\cdot,\textbf{X}} [(\textbf{K}+\sigma^2 I)^{-1} - (\tilde{\textbf{K}}+\sigma^2 I)^{-1}] \Vert_2 + \Vert (\textbf{K}_{\cdot,\textbf{X}} - \tilde{\textbf{K}}_{\cdot,\textbf{X}}) (\tilde{\textbf{K}}+\sigma^2 I)^{-1} \Vert_2\\
    &\leq \Vert \textbf{K}_{\cdot,\textbf{X}} \Vert_2 \Vert (\textbf{K}+\sigma^2 I)^{-1} - (\tilde{\textbf{K}}+\sigma^2 I)^{-1} \Vert_2 + \Vert \textbf{K}_{\cdot,\textbf{X}} - \tilde{\textbf{K}}_{\cdot,\textbf{X}} \Vert_2 \Vert (\tilde{\textbf{K}}+\sigma^2 I)^{-1} \Vert_2\\
    &\leq \Vert \textbf{K}_{\cdot,\textbf{X}} \Vert_2 \frac{\gamma_{n,m,L}}{\sigma^4} + \Vert \textbf{K}_{\cdot,\textbf{X}} - \tilde{\textbf{K}}_{\cdot,\textbf{X}} \Vert_2 \frac{1}{\sigma^2},
\end{align*}
where we used Lemma \ref{lemma:action-inverse-error} in the last inequality. Substituting this back into the main inequality, we get:
\begin{align*}
    &\Vert \textbf{K}_{\cdot,\textbf{X}} (\textbf{K}+\sigma^2 I)^{-1}\textbf{K}_{\textbf{X},\cdot}-\tilde{\textbf{K}}_{\cdot,\textbf{X}} (\tilde{\textbf{K}}+\sigma^2 I)^{-1}\tilde{\textbf{K}}_{\textbf{X},\cdot}\Vert_2 \\
    &\leq \frac{1}{\sigma^2} \Vert \textbf{K}_{\cdot,\textbf{X}}\Vert_2 \Vert \textbf{K}_{\textbf{X},\cdot}-\tilde{\textbf{K}}_{\textbf{X},\cdot}\Vert_2 + \left(\Vert \textbf{K}_{\cdot,\textbf{X}} \Vert_2 \frac{\gamma_{n,m,L}}{\sigma^4} + \Vert \textbf{K}_{\cdot,\textbf{X}} - \tilde{\textbf{K}}_{\cdot,\textbf{X}} \Vert_2 \frac{1}{\sigma^2}\right) \Vert \tilde{\textbf{K}}_{\textbf{X},\cdot} \Vert_2\\
    &= \frac{1}{\sigma^2} \Vert \textbf{K}_{\cdot,\textbf{X}}\Vert_2 \Vert \textbf{K}_{\textbf{X},\cdot}-\tilde{\textbf{K}}_{\textbf{X},\cdot}\Vert_2 + \frac{\gamma_{n,m,L}}{\sigma^4}\Vert \textbf{K}_{\cdot,\textbf{X}} \Vert_2  \Vert \tilde{\textbf{K}}_{\textbf{X},\cdot} \Vert_2 + \frac{1}{\sigma^2} \Vert \textbf{K}_{\cdot,\textbf{X}} - \tilde{\textbf{K}}_{\cdot,\textbf{X}} \Vert_2 \Vert \tilde{\textbf{K}}_{\textbf{X},\cdot} \Vert_2.
\end{align*}

Using Lemma \ref{lemma:test-train-kernel-matrix-error} and the fact that $\Vert \textbf{K}_{\textbf{X},\cdot}-\tilde{\textbf{K}}_{\textbf{X},\cdot}\Vert_2 \leq \max(\gamma_{T,m,L},\gamma_{n,m,L})$ and that $\textbf{K}_{\cdot, \textbf{X}} = \textbf{K}_{\textbf{X},\cdot}^\top$, we have $\Vert \textbf{K}_{\cdot,\textbf{X}} \Vert_2 = \Vert \textbf{K}_{\textbf{X},\cdot} \Vert_2$. Also, by assumption, $\Vert \textbf{K}_{\textbf{X},\cdot} \Vert_2 \leq \sqrt{Tn}M$. Using Lemma \ref{lemma:ski-test-train-bound}, we have $\Vert \tilde{\textbf{K}}_{\textbf{X},\cdot} \Vert_2 \leq \sqrt{Tn}mc^{2d}M$. Substituting these bounds, we get:

\begin{align*}
    &\Vert \textbf{K}_{\cdot,\textbf{X}} (\textbf{K}+\sigma^2 I)^{-1}\textbf{K}_{\textbf{X},\cdot}-\tilde{\textbf{K}}_{\cdot,\textbf{X}} (\tilde{\textbf{K}}+\sigma^2 I)^{-1}\tilde{\textbf{K}}_{\textbf{X},\cdot}\Vert_2 \\
    &\leq \frac{\sqrt{Tn}M}{\sigma^2} \max(\gamma_{T,m,L},\gamma_{n,m,L}) + \frac{\gamma_{n,m,L}}{\sigma^4}(\sqrt{Tn}M)(\sqrt{Tn} m c^{2d} M) + \frac{1}{\sigma^2} \max(\gamma_{T,m,L},\gamma_{n,m,L}) (\sqrt{Tn} m c^{2d} M) \\
    &= \frac{\sqrt{Tn}M}{\sigma^2} \max(\gamma_{T,m,L},\gamma_{n,m,L}) + \frac{\gamma_{n,m,L}}{\sigma^4}Tn m c^{2d} M^2 + \frac{\sqrt{Tn} m c^{2d} M}{\sigma^2} \max(\gamma_{T,m,L},\gamma_{n,m,L}).
\end{align*}

Finally, substituting this back into the original inequality, we obtain the desired bound:

\begin{align*}
    \Vert \boldsymbol{\Sigma}(\cdot)-\tilde{\boldsymbol{\Sigma}}(\cdot)\Vert_2 &\leq \gamma_{T,m,L} + \frac{\sqrt{Tn}M}{\sigma^2} \max(\gamma_{T,m,L},\gamma_{n,m,L})\\
    &+ \frac{\gamma_{n,m,L}}{\sigma^4}Tn m c^{2d} M^2 + \frac{\sqrt{Tn} m c^{2d} M}{\sigma^2} \max(\gamma_{T,m,L},\gamma_{n,m,L}).\\
    &=O\left(\frac{Tn^2mc^{4d}M^2+\sqrt{Tn}mc^{4d}M\max(T,n)}{m^{3/d}}\right).
\end{align*}
\end{proof}